\documentclass[nohyperref]{article}

\usepackage{microtype}
\usepackage{graphicx}
\usepackage{subcaption}
\usepackage{booktabs} 
\usepackage{microtype}
\usepackage[left=3cm,right=3cm,top=3cm,bottom=3cm]{geometry}
\usepackage{natbib}
\usepackage[T1]{fontenc}

\usepackage{hyperref}

\usepackage{amsmath}
\usepackage{amssymb}
\usepackage{mathtools}
\usepackage{amsthm}

\usepackage{authblk}
\usepackage[capitalize,noabbrev]{cleveref}

\theoremstyle{plain}
\newtheorem{theorem}{Theorem}[section]
\newtheorem{propo}[theorem]{Proposition}
\newtheorem{lemma}[theorem]{Lemma}
\newtheorem{corollary}[theorem]{Corollary}
\theoremstyle{definition}
\newtheorem{definition}[theorem]{Definition}
\newtheorem{assumption}{Assumption}

\theoremstyle{remark}
\newtheorem{remark}[theorem]{Remark}
\newtheorem{example}[theorem]{Example}

\usepackage[textsize=tiny]{todonotes}

\usepackage{tikz}
\usetikzlibrary{positioning}
\usetikzlibrary{shapes}
\usetikzlibrary{scopes}

\usepackage{graphicx}
\graphicspath{{./}}

\usepackage{enumitem}

\usepackage{mathrsfs}

\usepackage{blindtext}

\usepackage{dsfont} 

\global\long\def\esp{\mathbb{E}}%
\global\long\def\F{\mathcal{F}}%
\global\long\def\K{\mathcal{K}}%
\global\long\def\R{\mathbb{R}}%
\global\long\def\N{\mathbb{N}}%
\global\long\def\P{\mathbb{P}}%

\newcommand{\e}[1]{\mathbb{E}\left[#1\right]  } 
\newcommand{\V}[1]{\mathbb{V}\left[#1\right]  } 

\newcommand{\ind}{\mathds{1}} 
\newcommand{\norm}[1]{\left\Vert#1\right\Vert} 
\newcommand{\supp}[1]{\rm{Supp}\left(#1\right)} 
\newcommand{\M}{\mathcal{M}}

\newcommand{\mm}{\gamma}

\newcommand{\prob}[1]{\mathbb{P}\left(#1\right)} 

\newcommand{\cb}[1]{}
\newcommand{\al}[1]{}
\newcommand{\ad}[1]{}
\newcommand{\es}[1]{}

\usepackage[normalem]{ulem}

\definecolor{OliveGreen}{rgb}{0.24, 0.71, 0.54}
\definecolor{RoyalBlue}{rgb}{0.0, 0.47, 0.75}
\definecolor{BrickRed}{rgb}{0.77, 0.12, 0.23}
\definecolor{Vert}{RGB}{0,128,0}
\newcommand{\red}[1]{#1}

\newcommand{\orange}[1]{#1}

\author[,1]{Alexis Ayme\thanks{Corresponding author: \texttt{alexis.ayme@sorbonne\_universite.fr}}}
\author[1,2]{Claire Boyer}
\author[3]{Aymeric Dieuleveut}
\author[3]{Erwan Scornet}

\affil[1]{Sorbonne Université, CNRS, Laboratoire de Probabilités, Statistique et Modélisation (LPSM), F-75005 Paris, France}
\affil[2]{MOKAPLAN, INRIA Paris}
\affil[3]{CMAP, UMR7641, Ecole Polytechnique, IP Paris, 91128 Palaiseau, France}

\begin{document}
\title{Minimax rate of consistency for linear models with missing values}

\maketitle

\begin{abstract}
Missing values arise in most real-world data sets due to the aggregation of multiple sources and intrinsically missing information (sensor failure, unanswered questions in surveys...). 
In fact, the very nature of missing values usually prevents us from running standard learning algorithms. 
In this paper, we focus on the extensively-studied linear models, but in presence of missing values, which turns out to be quite a challenging task. 
Indeed, the Bayes rule can be decomposed as a sum of predictors corresponding to each missing pattern.
This eventually requires to solve a number of learning tasks, exponential in the number of input features, which makes predictions impossible for current real-world datasets. 
First, we propose a rigorous setting to analyze a least-square type estimator and establish a bound on the excess risk which increases exponentially in the dimension. Consequently, 
we leverage the missing data distribution to propose a new algorithm, and
derive associated adaptive risk bounds that turn out to be minimax optimal. 
Numerical experiments highlight the benefits of our method compared to state-of-the-art algorithms used for predictions with missing values.
\end{abstract}

\section{Introduction}

Missing values are more and more present as the size of datasets increases. These missing values can occur for a variety of reasons, such as sensor failures, refusals to answer poll questions, or aggregations of data coming from different sources (with different methods of data collection). 
There may be different processes of missing value generation on the same dataset, which makes the task of data cleaning difficult or impossible without creating large biases.  
In his leading work, \citet{RUBIN76} distinguishes three missing values scenarios:  Missing Completely At Random (MCAR),  Missing At Random (MAR), and Missing Not At Random (MNAR), depending on the links between the observed variables, the missing ones, and the missing pattern. 

In the linear regression framework, most of the literature focuses on \textit{parameter estimation}  \citep{little1992regression,Jones1996IndicatorAS}, using sometimes a sparse prior leading to the Lasso estimator \citep{LohWainwright12} or the Dantzig selector \citep{Rosenbaum2009}. 
Note that the robust estimation literature \citep{DalThompson2019,chen13} could be also used to handle missing values, as the latter can be reinterpreted as a multiplicative noise in linear models.  
Besides, \citet{sportisse2020debiasing} adapt and theoretically study the famous stochastic gradient algorithm for model estimation in online linear regression.

On the other hand, \textit{prediction} with missing values in a parametric framework -even under a linear model- is in fact not an easy task. 
Indeed, the prediction task is distinct from model estimation: estimated model parameters cannot be directly used to predict on a test sample containing missing values as well. 
As a matter of fact, the occurrence of missing data turns the linear regression problem into a semi-discrete one of very high complexity. 
Finally, establishing risk bounds -even without missing values- for random designs is already a challenge as studied in papers \citep{gyorfi2006distribution,audibert2011robust,dieuleveut2017harder} and more recently in \citep{mourtada2019exact}.

\textbf{Related work.} 
There is actually little work on prediction with missing values.
\citet{pelckmans2005handling} adapt the SVM classifier to the case of missing values.
\citet{josse2019consistency} study the consistency of imputation strategies prior to non-parametric learning methods.
Prediction under linear models has been studied in \citep{lemorvan:hal-02888867,le2020linear},  
by exploiting the peculiar pattern-by-pattern structure of the Bayes predictor (i.e.\ decomposable into predictors specific to each missing pattern), and estimating it when the input variables are assumed to be Gaussian. \citet{le2020linear} obtain risk bounds, that suffer from the curse of dimensionality, and are actually not compatible with their Gaussian assumption.

\textbf{Contributions.} In this paper, we study pattern-by-pattern predictors for regression with missing input variables. 
First, we provide a synthetic overview of all assumptions that allow to obtain a pattern-wise linear Bayes predictor, and we propose a detailed study on how these assumptions are related (\Cref{sec:Generalization}).
Second, we provide a distribution-free excess risk bound for a least-square estimator handling unbounded features (\Cref{sec:unboundedFeatures}), but suffering from the curse of dimensionality. 
We therefore introduce a novel thresholded estimator for which we establish an excess risk bound adaptive to the missing pattern distribution (\Cref{sec:main_result_regression}). 
The latter actually applies to all types of missing data (MCAR, MAR, MNAR) and is shown to be minimax optimal.
We exhibit three settings in which our bound is precisely evaluated, improving upon state-of-the-art results.
Finally, we experimentally illustrate our  method on three 
different simulation settings, outperforming existing competitors designed to handle missing values, both in terms of predictive performance and computational time (\Cref{sec:experiments}).
All the proofs of theoretical results can be found in the supplementary materials.

\textbf{Notations. } For $n\in \mathbb N$, we denote $[n] = \lbrace 1, \dots, n\rbrace$. We use $\lesssim$ to denote inequality up to a universal constant. We denote $a\wedge b= \min(a,b)$ and $a\vee b= \max(a,b)$.


\section{Typology of missing value and its consequence on the Bayes predictor}\label{sec:Generalization}

\subsection{Setting}

In a context of regression, we observe $n\in \mathbb N$ input/output observations $(X_i,Y_i)_{i\in [n]}$, i.i.d.~copies of a generic pair $(X,Y)\in \mathbb{R}^d \times \mathbb{R}$, assuming that the underlying model linking $Y$ to $X$ is linear.

\begin{assumption}[Linear Model]\label{ass:model_lin} 
$Y = \beta_0+\beta^{\top}X +\epsilon,$ with a Gaussian noise $\epsilon\sim\mathcal{N}(0,\sigma^2)$ independent of $X$.
\end{assumption}

The (unknown) model parameters are therefore $(\beta_0,\beta)\in\mathbb{R}^{d+1}$. Although standard linear regression is a well-understood problem in statistics, 
we consider here that only a fraction of the components of $X$ is available: to the data $X\in\mathbb{R}^d$ one associates the missing values pattern $M\in \{0,1\}^d$, such that $M_j=1$ if and only if  $X_j$ is missing. Let $\M=\{0,1\}^d$ be the set of missing values patterns. 
For $m\in\mathcal{M}$, we denote by $\rm{obs}(m)$ (resp. $\rm{mis}(m)$) the set of indexes of the observed variables (resp. the missing variables) and $X_{\rm{obs}(m)}$ (resp. $X_{\rm{mis}(m)}$) the vector of observed components (resp. unobserved components) of $X$.
Thus, under a linear model with missing covariates, our goal is to predict $Y$ given $ \left( X_{\rm{obs}(m)},M\right)$, denoted $Z$ in the sequel.

\subsection{Bayes predictor}

The Bayes predictor for the quadratic loss can be decomposed according to the possible missing data patterns, as

\begin{align}
        \notag
		f^{\star}(Z) &= \e{Y|Z}= \e{Y|X_{\rm{obs}(m)},M}\\
		\notag 
		&=\sum_{m\in\M}f_m^{\star}(X_{\rm{obs}(m)})\ind_{M=m},
		\label{eq:bayes_m}
\end{align}
where $f_m^{\star}(X_{\rm{obs}(m)}):=\e{Y|X_{\rm{obs}(m)},M=m}$ can be seen as the Bayes predictor conditionally on the event $``M=m"$.
Under Assumption \ref{ass:model_lin}, $f_m^\star$ can be written as 
 \begin{equation*}
     f_m^{\star}(X_{\rm{obs}(m)})= \beta_0 +\beta_{\rm{obs}(m)}^{\top}X_{\rm{obs}(m)}
     +\beta_{\textrm{mis}(m)}^{\top}\e{X_{\rm{mis}(m)}|X_{\rm{obs}(m)},M=m}.
 \end{equation*}
Thus, $f_m^{\star}$ remains linear in the observed variables $X_{obs}$, provided that $x\mapsto \e{X_{\rm{mis}(m)}|X_{\rm{obs}(m)}=x,M=m}$ is a linear function. 
This is not always true as shown in the following example. 

\begin{example}\label{ex:nonlinear}
Let $Y=X_1+X_2+X_3+\epsilon$, where $X_3=X_2e^{X_1}$. Then 
\begin{equation*}
    f_{(0,0,1)}^\star(X_1, X_2)=X_1+X_2+X_2e^{X_1},  
\end{equation*}
where $m=(0,0,1)$ is the missing value pattern where only $X_1$ and $X_2$ are observed. Despite Assumption \ref{ass:model_lin}, the predictor $f_{(0,0,1)}^\star$ is not linear in the observed covariates, due to the non-linear link between the observed variables $X_1,X_2$ and the missing one $X_3$.
Therefore, linear regression with missing data is hard to analyze without any additional assumptions on the joint distribution $(X,M)$.
\end{example}

\subsection{Data scenarios}\label{sec:dataScenario}
There exist two main approaches for modelling the joint distribution of $X$ and $M$: selection models \citep{heckman} and pattern-mixture ones \citep{little1993pattern}.
    

\paragraph{Selection models.} They rely on the following factorization of the joint distribution $
\P\left(X,M\right)=\P(X)\P(M|X).
$
Therefore, in selection models, one specifies the distributions of $X$ (the most common ones being  Assumptions~\ref{ass:indepX} and \ref{ass:Gaussian} below) and $M|X$ (Assumptions~\ref{ass:MCAR}, \ref{ass:MAR} or \ref{ass:MNAR} below). 
\begin{assumption}[Independent covariates]\label{ass:indepX}
The covariates $\{X_j\}_{j\in[d]}$ are mutually independent.
\end{assumption}
\begin{assumption}[Gaussian covariates]\label{ass:Gaussian}
There exist $\mu\in \mathbb{R}^d$ and $\Sigma\in \mathbb{R}^{d\times d}$ such that $X\sim\mathcal{N}\left(\mu,\Sigma \right)$.
\end{assumption}
Note that this latter assumption excludes the pathological \cref{ex:nonlinear}. 
Regarding the distribution of $M|X$, \citet{RUBIN76} introduces the three following missingness mechanisms.
\begin{assumption}[Missing Completely At Random - MCAR]\label{ass:MCAR}
    For all $m\in\M$, $\prob{M=m|X}=\prob{M=m}$.
\end{assumption}

\begin{assumption}[Missing At Random - MAR]\label{ass:MAR}
    For all $m\in\M$,  $\prob{M=m|X}=\prob{M=m|X_{\rm{obs}(m)}}$.
\end{assumption}

\begin{assumption}[Missing Non At Random - MNAR]\label{ass:MNAR}
    The missing pattern $M$ depends on the full vector $X$ (thus, on the observed and missing entries).
\end{assumption}

To illustrate these scenarios, consider the simple situation of a survey
with two variables, \textit{Income} and \textit{Age}, with missing values only on the \textit{Income} variable. 
The MCAR setting (Assumption \ref{ass:MCAR}) holds when the missing values are independent of any value (e.g. respondents
have forgotten to fill the form). The MAR situation (Assumption \ref{ass:MAR}) is verified when missing values on \textit{Income}
depend on the values of \textit{Age} (e.g. younger respondents would be less inclined to reveal their
income). The MNAR scenario (Assumption \ref{ass:MNAR}) allows the occurrence of the missing values on \textit{Income} to
depend on the values of the income itself (e.g.\  poor and rich respondents would be less inclined to reveal their income).
A particular case of the last example consists in considering that the missingness mechanism for a given variable is only dictated by its underlying value.
\begin{assumption}[Gaussian Self-Masking]\label{ass:selfmasking}  
For all $m\in\M$, $\P(M=m|X)=\prod_{j=1}^d\P\left(M_j=m_j|X_j\right)$ and for $ j\in[d]$, 
\begin{equation*}
    \P(M_j=1|X_j)\propto\exp\left(-\frac{1}{2}\frac{(X_j-\widetilde{\mu}_j)^2}{\widetilde{\sigma}_j^2}\right).
\end{equation*}
\end{assumption}

\paragraph{Pattern-mixture models. }  Such models rely on the following factorization of the joint distribution 
$
\P\left(X,M\right)=\P(M)\P(X|M).
$
Therefore, in pattern-mixture models, one specifies the distributions of $M$ and $X|M$: 
one can therefore appeal to the Gaussian pattern mixture model (GPMM).
\begin{assumption}[Gaussian Pattern Mixture Model-GPMM]  
	\label{ass:ass4.1} 
	For all  $m\in\mathcal{M}$, $ X|(M=m)\sim\mathcal{N}(\mu^{(m)},\Sigma^{(m)}).$
\end{assumption}




\subsection{Links between Gaussian PMM \& selection models.} 
In this subsection, we investigate the links between the different sets of assumptions of~\Cref{sec:dataScenario}, summarized in~\Cref{fig:GMM}.
\begin{figure}[h]
    \centering
    \begin{tikzpicture}[every node/.style={scale=1.1},scale=0.9]
\node[draw] (A) at (1,0) {G+MAR (\ref{ass:Gaussian} and \ref{ass:MAR})};
\node[draw] (B) at (7.5,0) {MAR (\ref{ass:MAR})};
\node[draw] (C) at (0,-3) {G+MCAR (\ref{ass:Gaussian} and \ref{ass:MCAR})};
\node[draw,very thick] (D) at (5,-3) {GPMM (\ref{ass:ass4.1})};
\node[draw] (F) at (9,-3) {MNAR (\ref{ass:MNAR})};
\draw [->,>=latex] (A) ->  (B);
\draw[->,>=latex] (C) -> (A);
\draw[->,>=latex] (C) -> (D);
\draw[dashed,->] (D) to[bend left]node[anchor=north] { Example \ref{ex:MCAR_GMM} } (C);
\draw[dashed,->] (D) to[bend left]node[anchor=south] {Example \ref{ex:GMMAR}$\qquad\qquad\;$} (B);
\draw[->,>=latex] (D) to node[draw, red, sloped, cross out, line width=.5ex, minimum width=1.5ex, minimum height=1ex, anchor=center]{}node[anchor=north] { $\qquad\qquad\qquad$Example \ref{ex:GMMNAR}}  (B);
\draw[->,>=latex] (A) to node[draw, red, sloped, cross out, line width=.5ex, minimum width=1.5ex, minimum height=1ex, anchor=center]{}node[anchor=north] {$\qquad$  Example \ref{ex:MARnoGMM}$\qquad\qquad\qquad\qquad$}  (D);
\draw[->,>=latex] (B) -> (F);
\draw[->,>=latex] (D) -> (F);
\end{tikzpicture}
    \caption{Links between Gaussian pattern mixture models (GPMM) and Gaussian selection models. Solid arrows correspond to inclusions, dotted (resp. crossed) arrows illustrate a partial inclusion (resp. non-inclusion).}
    \label{fig:GMM}
\end{figure}
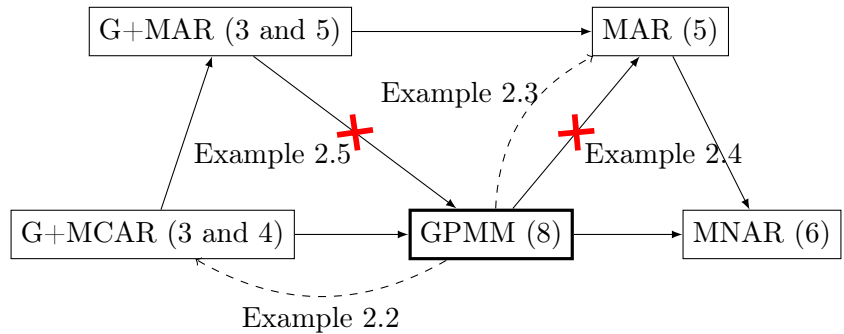

First, we remark that GPMM may implicitly encode for M(C)AR and MNAR scenarios. 
\begin{example}[From GPMM to MCAR]\label{ex:MCAR_GMM}
Consider a GPMM such that there exist $\mu$ and $\Sigma$ such that $\mu^{(m)}=\mu$ and $\Sigma^{(m)}=\Sigma$ for all $m\in\M$. One can show that the latter is necessary and sufficient to get a MCAR dataset (Assumption \ref{ass:MCAR}) with Gaussian covariates $X$ (Assumption \ref{ass:Gaussian}). 
\end{example}

\begin{example}[From GPMM to MAR]\label{ex:GMMAR}
Consider a subset of always observed variables indexed by $J\subset [d]$  (i.e.\ $\P(M_j=0)=1$ for $j \in J$) and a GPMM such that 
\[
\begin{cases}
\mu^{(m)}_J & =\mu_m \, \,\sim \mathcal{U}([-1,1]^{|J|}) \\
\mu^{(m)}_{J^c}&=\mu \in \mathbb{R}^{|J^c|} \, \, \text{(fixed)}
\end{cases}
\]
and
\[
\begin{cases}
\Sigma_{J,J}^{(m)} & =\Sigma_{m}  \in \mathbb{R}^{|J|\times |J|}\\
\Sigma_{J^{c},J^{c}}^{(m)} & =\Sigma  \in \mathbb{R}^{|J^c|\times |J^c|} \, \, \text{(fixed)}\\
\Sigma_{J^{c},J}^{(m)} & =0,
\end{cases}
\]
where $\Sigma_{m}\in\R^{J\times J}$ can depend on $m$. 
In such a case, for all $m\in \mathcal{M}$, 
$
    \P\left(M=m|X\right)= \P\left(M=m|X_J\right)
$
with $X_J$ always observed, thus the missing mechanism can be qualified of MAR. 
Furthermore, note that as soon as there exist $m,m'\in\M$, such that $\mu_m \neq\mu_{m'}$ or $\Sigma_m\neq\Sigma_m'$ then the dataset is ensured not to be MCAR. 
\end{example}

\begin{example}[From GPMM to MNAR]\label{ex:GMMNAR}
Consider a GPMM  such that for all $m\in \mathcal{M}$, $\Sigma^{(m)}=I_d$ and $\mu^{(m)}$ is  uniformly drawn at random in $[-1,1]^d$. In such a case, the missing mechanism is MNAR and almost surely not MAR.
\end{example}

Note that, Gaussian linear models with MAR missing values (Assumptions \ref{ass:Gaussian} and \ref{ass:MAR}) are not necessarily included in Gaussian pattern mixture models (Assumption \ref{ass:ass4.1}). This is in particular highlighted by the following example.

\begin{example}[G+MAR $\not \subseteq$ GPMM]\label{ex:MARnoGMM}
Let $(X_1,X_2)\sim\mathcal{N}\left(0,I_2 \right)$ such that $X_1$ is always observed ($M_1=0$) and $X_2$ is observed if and only if $X_1\leq 0$ ($M_2=\ind_{X_1>0}$). This corresponds to a linear Gaussian model (Assumption \ref{ass:Gaussian}), with a MAR missing variable $X_2$ (Assumption \ref{ass:MAR}) since missing values on $X_2$ only depend on $X_1$ which is  always observed.
However, this cannot be a GPMM (Assumption \ref{ass:ass4.1}) as the distribution of $X|(M=(0,1))$ is supported on a half space (preventing $X|M$ from being Gaussian). 
\end{example}




\subsection{Linearity of the Bayes predictor}
In this subsection, we give an overview of  the properties that ensure the 
linearity of $f_m^{\star}$. 
\begin{definition}
\label{def:linearity_bayes}
Consider the vector space of linear predictors in the observed variables
i.e.\ $f\in\F_b$ if $f(.,m)$ is linear for all $m\in\M$. The dimension of $\mathcal{F}_b$ is $p:=2^{d-1}(d+2)$. 
\end{definition}


\begin{propo}\label{prop:f_mLineaire}[\citealt{le2020linear,lemorvan:hal-02888867}, resp.\ Prop. 4.1 and Prop. 2.1]
{Assume one of the following hypotheses
\begin{enumerate}[noitemsep,topsep=0pt]
    \item Gaussian covariates  with M(C)AR mechanisms (Assumption \ref{ass:Gaussian} and (\ref{ass:MCAR} or \ref{ass:MAR})),
    \item Gaussian covariates with Gausian Self-Masking mechanisms (Assumption \ref{ass:Gaussian} and \ref{ass:selfmasking}),
    \item Gaussian Pattern Mixture Model (Assumption \ref{ass:ass4.1}),
    \item Independent covariates (Assumption~\ref{ass:indepX}). 
\end{enumerate}
}
 Then $f^\star\in\mathcal{F}_b$ i.e. for all $m\in\M$ there exist $\delta^{(m)}_0 \in \R$ and $\delta^{(m)}\in\R^{|m|}$ such that
\begin{equation*}
    f_m^{\star}(X_{\rm{obs}(m)})=\delta^{(m)}_0 +\left( \delta^{(m)}\right)^{\top}X_{\rm{obs}(m)}.
\end{equation*}
 
\end{propo}
Proposition \ref{prop:f_mLineaire} is summarized in Figure \ref{fig:assumption_relation}: there is indeed a wide variety of possible assumptions such that $f^\star\in\F_b$. 
Note that the covariates independence (Assumption~\ref{ass:indepX}) allows to get the Bayes rule linearity beyond Gaussian models.
Furthermore, Assumptions~\ref{ass:indepX}, \ref{ass:selfmasking}, and \ref{ass:ass4.1} may not only include MAR but MNAR scenarios as well, the latter known to be challenging in an inference setting.  However, this comes at the cost that the dimension $p$ of $\mathcal{F}_b$  grows exponentially with the ambient dimension $d$ ($p=2^{d-1}(d+2)$). 

\begin{figure}[h]
    \centering
    \begin{tikzpicture}[every node/.style={scale=1.1},scale=0.9]
\node[draw] (A) at (0,0) {G+ MCAR (\ref{ass:Gaussian} and \ref{ass:MCAR})};
\node[draw] (B) at (6,0) {G+ MAR (\ref{ass:Gaussian} and \ref{ass:MAR})};
\node[draw,fill=gray!50] (C) at (11,-2) {GSM (\ref{ass:selfmasking})};
\node[draw,fill=gray!50] (D) at (0,-2) {GPM (\ref{ass:ass4.1})};
\node[draw,fill=gray!50] (E) at (6,-4) {Independent covariates (\ref{ass:indepX})};
\node[draw,very thick] (F) at (6,-2) {$f_m^\star$ linear for all $m\in\M$};
\draw[->,>=latex] (A) -> (B);
\draw[->,>=latex] (A) -> (D);
\draw[->,>=latex] (C) -> (F);
\draw[->,>=latex] (B) -> (F);
\draw[->,>=latex] (D) -> (F);
\draw[->,>=latex] (E) -> (F);
\end{tikzpicture}
    \caption{Links between the different assumptions to get the linearity of the Bayes rule as in~\cref{prop:f_mLineaire}.  
    Scenarios that may contain MNAR cases are depicted in gray. }
    \label{fig:assumption_relation}
\end{figure}
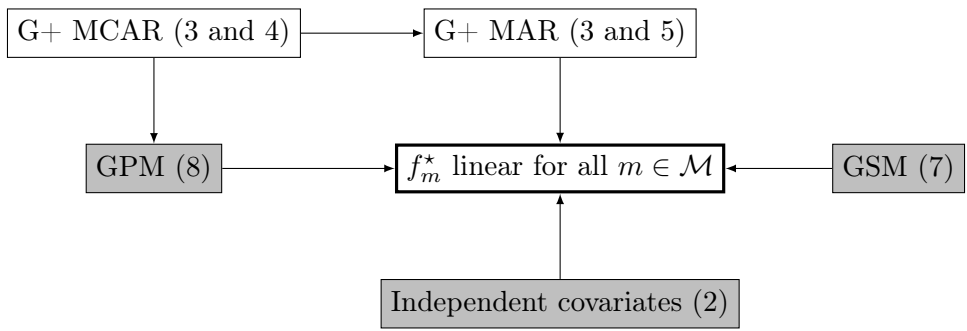

\section{A distribution-free bound on excess risk}\label{sec:unboundedFeatures}
In the framework of missing values,  let the excess risk be
\begin{align}
    \label{eq:def_risk_integrated}
    \mathcal{E}\left(\widehat{f}\right):=\esp\left[\left.\left(\widehat{f}(Z) - f^\star(Z)\right)^2  \right| \mathcal{D}_n \right],
\end{align}
and its integrated version $\mathbb{E}[\mathcal{E}(\widehat{f})]$.
This quantity measures the quality of performance made by a prediction function $\widehat{f}$ compared to the optimal predictor $f^\star$. 
\citet{le2020linear}  propose a missing-pattern-distribution-free control on the integrated risk scaling in ${d 2^{d}/n}$ for the least-square estimator, requiring  (i) the Bayes predictor to be linear in the observed variables (Definition~\ref{def:linearity_bayes}) and (ii) the covariates boundedness. Unfortunately,~\cref{prop:f_mLineaire} highlights the convenience of Gaussian covariates to ensure the linearity of the Bayes predictor, and yet they are incompatible with (ii). Thus, the result of \citep{le2020linear} is valid under some contradictory assumptions.
We intend to fill this gap by providing a unified and more general analysis to include the case of unbounded covariates (Assumption~\ref{ass:bornitude} below) and 
where the Bayes predictor is assumed to be regular without being explicitly linear (Assumption~\ref{ass:Lipsch}). 
%

\begin{assumption}[Sub-Gaussian covariate] \label{ass:bornitude}
	There is a positive constant $\mm$ such that for all $j\in[d]$, $\e{X_j}^2\leq \mm$ 
	and $X_j-\e{X_j}$ is $\mm$-sub-Gaussian, that is
	\begin{equation}
	    \forall t>0,\qquad \prob{\left|X_j-\e{X_j}\right|>t}\leq e^{-\frac{t^2}{2\mm}}.
	\end{equation}
\end{assumption}

\begin{assumption}[Lipschitz]\label{ass:Lipsch}
There exists $B>0$ such that for all $m\in\M$, $f_m^\star$ is $B$-Lipschitz for the $\ell^\infty$-norm, and $|f_m^\star (0)|\leq B$. 
\end{assumption}
According to Assumption~\ref{ass:Lipsch}, one can control the $\ell^\infty$-norm of Bayes predictors on an $\ell^\infty$-ball. For instance, this assumption is easily verified when Bayes predictors are linear functions. Since covariates are assumed to be unbounded (\Cref{ass:bornitude}), one should consider the set 
\begin{equation}\label{eq:KD}
	\K_D:=\left\lbrace (x_{\rm{obs}(m)},m),\quad\left \Vert x_{\rm{obs}(m)}\right\Vert_\infty\leq D \right\rbrace ,
\end{equation}
for some $D>\sqrt{\mm}$, which consists in taking the covariates with all observed components in an $\ell^\infty$-ball of radius $D$. 

Under~\cref{ass:bornitude}, 
an observation $Z$ falls into the bounded set $\K_D$ with high probability (see~\cref{lem:p_k}). One can then adapt the results in \cite{gyorfi2006distribution,audibert2011robust} when $X$ is on a bounded set to the sub-Gaussian case. To do so, consider the modified least-squares ($D$-LS) estimator taking into account only the observations falling into $\K_D$:
 \begin{equation}\label{eq:OLS_KD}
			\hat{f}^{(D\textrm{-LS})}\in\arg\!\min_{f\in\F_b}\sum_{Z_i\in\K_D}(f(Z_i)-Y_i)^2 
\end{equation}
if $\K_D \neq \emptyset $, and $\hat{f}^{(D\textrm{-LS})}=0$ otherwise. {Computing $\hat{f}^{(D\textrm{-LS})}$ amounts to perform one ordinary least-square procedure per missing pattern (as $\F_b$ is composed of functions that are linear on each missing pattern).} 
Finally, for technical purposes, to ensure that the prediction is bounded, 
we consider the \textit{clipped} estimator at level $L$, $T_L \hat{f} := (-L)\vee \hat{f}\wedge L$.

\begin{theorem}\label{thm:Risque_Complet}
Under Assumptions \ref{ass:bornitude} and \ref{ass:Lipsch}, choosing $D=\sqrt{\mm}(1+\sqrt{\mm\log(n)})$, and $L=(D+1)(B+1)$ leads to 
\begin{equation}
    \begin{split}
        \esp\left[\mathcal{E}\left(T_{L}\hat{f}^{(D\textrm{-}\rm{LS})}\right)\right]\lesssim  {(\log(n)+1)}
{\left(\sigma_{\rm{na}}^{2}\vee L^2\right)}
{2^{d}\frac{d}{n}}\\
+ A_{\F_b},
    \end{split}
\end{equation}
where $\hat{f}^{(D\textrm{-}\rm{LS})}$ is the estimator defined in \eqref{eq:OLS_KD}, and
\begin{equation}
    \begin{cases}
        \sigma_{\text{na}}^{2}:=\sup_{z\in\supp Z }\mathbb{V}\left[Y|Z=z\right]\\
        A_{\F_b}:=\inf_{f\in\F_b}\e{\left(f(Z)-f^\star(Z)\right)^2}.
    \end{cases}
\end{equation}

\end{theorem}

Theorem \ref{thm:Risque_Complet} is the first theoretical result that provides a control on the excess risk of a least-square-type predictor under very general assumptions on the input variables distribution and without any assumption on the missing pattern distribution.
This result only relies on concentration and regularity arguments. 
Note that leaving the approximation error aside, the obtained upper bound is the multiplication of three terms. The first factor ${(\log(n)+1)}$ is due to \citep[Theorem 11.3]{gyorfi2006distribution}  on which our result is built upon. 
The second factor ${\sigma_{\text{na}}^{2}\vee L^{2}}$ should be seen as a tight bound for $\esp[Y^2]$, which corresponds to the risk of the trivial predictor (predicting $0$ for any value of $Z$).
{Note that the coefficient $L$ logarithmically depends on $n$: the truncation of the predictor should be less stringent with an increasing number of observations.}
The rate of convergence is eventually dictated by the factor ${2^{d}\frac{d}{n}}$, which remains problematic as it grows exponentially with the dimension. It reflects the fact that a different regression model is required for each missing value pattern. 
Overall, the bound ensures that when $n>d2^d$, the least-square predictor is better than the zero one.
This curse of dimensionality is the price to pay as the result is valid for any missing pattern distribution.

In the framework of \cref{prop:f_mLineaire} (cases 1-3), ~\Cref{ass:bornitude} and \ref{ass:Lipsch} trivially hold and the Bayes predictor is ensured to be linear. This wipes the approximation error out in~\cref{thm:Risque_Complet} as underlined in the following result.
\begin{corollary}\label{cor:freeBoundAndLin}
    Under Assumptions [\ref{ass:Gaussian} and (\ref{ass:MCAR} or \ref{ass:MAR})] or \ref{ass:ass4.1}, with the same choice of $D$ and $L$ as in Theorem~\ref{thm:Risque_Complet} with $B=  \max_{m\in\M}\max[|\delta_{0}^{(m)}|,\Vert \delta^{(m)}\Vert _{1}]$, we have
\[
\esp\left[\mathcal{E}\left(T_L\hat{f}^{(D\textrm{-}\rm{LS})}\right)\right]\lesssim  {(\log(n)+1)}{\left(\sigma_{\text{na}}^{2}\vee L^2\right)}{2^{d}\frac{d}{n}},
\] 
where the sub-Gaussian parameter $\gamma$ in $L,D$ is
\begin{equation*}
    \gamma = \begin{cases}
     \max\limits_{j\in[d]}\esp\left[X_{j}^{2}\right] & \textrm{(Assumption \ref{ass:Gaussian})},\\
     
     \max\limits_{m\in \mathcal{M} \atop j\in[d]}\esp\left[X_{j}^{2}|M=m\right] & \textrm{(Assumption \ref{ass:ass4.1})}.
    \end{cases}
\end{equation*}
\end{corollary}
To ease the readability, we define when possible
\begin{equation}\label{def:a_n}
    a_n:=  {(\log(n)+1)}{\left(\sigma_{\text{na}}^{2}\vee L^2\right)},
\end{equation} 
which logarithmically grows with $n$ and depends on the distribution of $(X,Y)$. 

\section{Main result: an excess risk bound adaptive to the missing pattern distribution}
\label{sec:main_result_regression}

The error bound obtained in \Cref{thm:Risque_Complet} holds for any missing pattern distribution.
For instance, when all the $2^d$ missing patterns are equiprobable, the bound of  \Cref{thm:Risque_Complet} appears sharp -as one should actually perform $2^d$ ``independent" regressions- and then suffers from the curse of dimensio\-na\-li\-ty. 
However, this bound is pessimistic when some missing patterns are not observed or, more generally, when the missing pattern distribution is non-uniform, i.e. of low entropy. In this section, we leverage the distribution of the missing patterns in order to derive better theoretical bounds compared to \Cref{thm:Risque_Complet}. To this end, we propose a refined version of the predictor introduced in \cref{eq:OLS_KD}.

\subsection{Regression only on high frequency missing patterns }\label{subsec:reg_rule}

For any missing pattern $m\in\M$, we denote $E_{m}=\left\{ i\in[n],M_{i}=m\right\} $
and $\mathcal{D}_{n}^{(m)}=((X_{i,\rm{obs}(m)},Y_{i}))_{i\in E_{m}}$  respectively the observation indices and the sub-sample with missing pattern $m$.
For any $m \in \M$, we build an estimator $\widetilde{f}_{m}$ of $f_m^\star$ as 
\begin{equation}
\label{eq:OLSechantillon2}
\widetilde{f}_{m}\in\arg\min_{f\in\F_m}\sum_{i\in \overline{E}_m}(f(X_i)-Y_i)^2
\end{equation}
if $\overline{E}_m :=\left\{ i\in E_{m},\left\Vert X_{\rm{obs}(m)}\right\Vert _{\infty}\leq D\right\}$ is non-empty, and $\widetilde{f}_{m}=0$ otherwise.
The global predictor is then obtained by combining the previous pattern-by-pattern predictors for all patterns $m\in \M$ that appear with a frequency $\widehat{p}_{m}:=\frac{|E_{m}|}{n}$ larger than a threshold $\tau \in [0,1]$, 
\begin{equation}\label{eq:reg_with_tau}
\widehat{f}^{(\tau)}(Z)=\sum_{m\in\M}\widetilde{f}_{m}\left(X_{\rm{obs}(m)}\right)\ind_{\widehat{p}_{m}>\tau}\ind_{M=m}.
\end{equation}

Contrary to the naive estimator $\hat{f}^{(D\textrm{-LS})}$ defined in \eqref{eq:OLS_KD}, computing $\widehat{f}^{(\tau)}$ may not require to perform up to $2^d$ linear regressions. Indeed, linear regressions are only computed  for patterns with a frequency larger than the threshold $\tau$. 
This new predictor \eqref{eq:reg_with_tau} enjoys the following risk bounds.

\begin{theorem}\label{thm:borne_reg_threhold}
Under the same assumptions as in \Cref{thm:Risque_Complet}, for any  $\tau\geq 1/n$,  the generalization bound for the predictor $\widehat{f}^{(\tau)}$ defined in \eqref{eq:reg_with_tau}, reads as
\begin{equation}\label{eq:reg_bound_tau1}
\begin{split}
     \esp\left[\mathcal{E}\left(T_L\widehat{f}^{(\tau)}\right)\right]\lesssim  a_n{\left(1\vee\frac{d}{n\tau}\right)\mathfrak{C}_p(\tau)} + A_{\F_b}.
\end{split}
       \end{equation}
   where $a_n$ is defined in \eqref{def:a_n},  and with the \emph{missing patterns distribution complexity} $\mathfrak{C}_p(\tau)$ defined by
\begin{equation}\label{def:E}
 \mathfrak{C}_p(\tau):=\sum_{m\in\M}p_m\wedge \tau.  
\end{equation}
The upper bound in Inequality \eqref{eq:reg_bound_tau1} is minimal for the choice $\tau = d/n$ which leads to 
\begin{align}\label{eq:reg_bound_tau_opti}
\esp\left[\mathcal{E}\left(T_L\widehat{f}^{(d/n)}\right)\right]\lesssim  a_n{\mathfrak{C}_{p}\left(\frac{d}{n}\right)}+ A_{\F_b}.
\end{align}



\end{theorem}
Theorem \ref{thm:borne_reg_threhold} is the first result controlling the excess risk of a pattern-by-pattern least-square-type predictor with a bound depending on the missing pattern distribution through the complexity $\mathfrak{C}_{p}$, and holds for any type of missing patterns. 
\Cref{thm:borne_reg_threhold} improves over \cref{thm:Risque_Complet}, as the pattern distribution complexity $\orange{\mathfrak{C}_p}$ is a lower bound of $\red{2^dd/n}$. 
Note that choosing $\tau=d/n$ is relevant only in the case where $d<n$ (otherwise, the proposed predictor is the zero one).
The adaptivity of $\mathfrak{C}_p$ to the missing pattern distribution is illustrated in the following examples.

\subsection{Examples}\label{subsec:exemple_borne}

In this subsection, we compute the quantity $\mathfrak{C}_{p}\left(\frac{d}{n}\right)$, driving the bound obtained in~\cref{thm:borne_reg_threhold},  for different missing data settings.
We focus on the case
$d\leq n\leq d2^{d}$, i.e.\ when we have enough
observations for statistical guarantees in standard linear regression (w/out missing values) but not enough when missing values occur (setting of~\cref{thm:Risque_Complet}.) 

\subsubsection{Example 1: Few frequent missing patterns}
One can actually write another characterization of the complexity $\mathfrak{C}_p$, as precised in the following lemma. 
\begin{lemma} 
\label{lem:Cp_characterization}
For any distribution $p$ on the missing patterns
\begin{equation*}
    \mathfrak{C}_p\left(\frac{d}{n}\right)=\inf_{\mathcal{B}\subset\M}\left\{{\rm{Card}}(\mathcal{B})\frac{d}{n}+\P\left(M\in \mathcal{B}^c\right) \right\},
\end{equation*}
where $\P\left(M\in \mathcal{B}^c\right) = \sum_{m\in \mathcal{B}^c} p_m$.
\end{lemma}
The proof can be found in Appendix \ref{proof:Cp_characterization}. To illustrate this lemma, consider a subset $\mathcal{B}\subset\mathcal{M}$ of small cardinality $| \mathcal B|$,
so that only missing patterns in $ \mathcal B$ are very frequent and that the other missing patterns occur with a residual probability $\delta=\P\left(M\in \mathcal{B}^c\right)$. \Cref{lem:Cp_characterization} entails that
\begin{equation}\label{eq:adaptability}
    \mathfrak{C}_p \left(\frac{d}{n}\right) \leq | \mathcal B|\frac{d}{n}+\delta.
\end{equation}
and thus by \Cref{thm:borne_reg_threhold},
\begin{equation}
\label{eq:adaptability_ex1}
     \esp\left[\mathcal{E}\left(T_L\hat{f}^{(d/n)}\right)\right]\lesssim a_n | \mathcal B|\frac{d}{n}+a_n\delta+ A_{\F_b}.
 \end{equation}
This bound clearly improves upon \Cref{thm:Risque_Complet}, as the complexity is now controlled by $| \mathcal B|\frac{d}{n}$ instead of $2^d\frac{d}{n}$.
This bound reflects the good learning ability of the regressor $\hat{f}^{(d/n)}$ when there are few frequent missing patterns.

Note that \Cref{lem:Cp_characterization} applies to any missing data mechanisms. In particular, MCAR, MAR and MNAR scenarios can be exemplified through the setting developed in this section, so that the upper bound \eqref{eq:adaptability_ex1} is very generic. The next two examples make use of this bound in two more specific scenarios, resulting in even more informative bounds.

\subsubsection{Example 2: The Bernoulli model} \label{sec:bernoulli}

Assume that the distribution $p$ of missing value patterns is $p=\mathcal{B}(\epsilon_{1})\otimes\cdots\otimes\mathcal{B}(\epsilon_{d})$ for $\epsilon_{j}\in\left[0,1\right]$ with 
$j\in[d],$ so that components $(M_j)_j$ are independent and of distribution $M_j \sim \mathcal{B}(\epsilon_j)$. 
The model is said homogeneous when $\epsilon_{1}=\epsilon_{2}=\cdots=\epsilon_{d}=\epsilon\in\left[0,1\right]$, and heterogeneous otherwise.
Note that in such a setting, the missing mechanisms can be still of MCAR, MAR or MNAR nature.

Consider a homogeneous Bernoulli model with $\epsilon< 1/2$.
Consequently, the most frequent patterns are those with the least missing values.
For a given $s\in[d]$, define $\mathcal B_s$ the set of missing patterns with less than $s$ missing values. Therefore, \Cref{eq:adaptability} reads as
\begin{equation}\label{eq:adaptabilityBernoulli}
    \mathfrak{C}_p \left(\frac{d}{n}\right) \leq | \mathcal B_s|\frac{d}{n}+\delta_s,
\end{equation}
where $\delta_s= \P\left(M\in \mathcal{B}_s^c\right)$ is the probability of having a pattern with more than $s$ missing values. 
Controlling each of these terms gives the following lemma. 
\begin{lemma}\label{lem:bernoulliOracle}
Under a homogeneous Bernoulli model with proportion $\epsilon$ of missing data, one has
\begin{equation*}
    \mathfrak{C}_p \left(\frac{d}{n}\right) \leq \inf_{s\in[d]}\left(\frac{d}{n}+\epsilon^s\right)\left(\frac{ed}{s}\right)^s.
\end{equation*}
\end{lemma}
One can then obtain a version of~\cref{thm:borne_reg_threhold} in the case of a Bernoulli model, by optimizing $s$ in Lemma \ref{lem:bernoulliOracle}.
  
\begin{propo}\label{prop:BernoulliEBound}
Under the assumptions of~\cref{thm:Risque_Complet},
\begin{equation*}
\begin{split}
  \esp\left[\mathcal{E}\left(T_L\hat{f}^{(d/n)}\right)\right]\lesssim a_n \left(\frac{ed}{s_{\epsilon}\left(d/n\right)}\right)^{s_{\epsilon}\left(d/n\right)}\frac{d}{n}+ A_{\F_b},
\end{split}
\end{equation*}
with
  $  s_{\epsilon}\left(d/n\right):=1 \vee \left \lfloor\frac{\log\left(\frac{n}{d}\right)}{\log(\epsilon^{-1})}\right\rfloor \wedge d.$
\end{propo}
Here, $s_{\epsilon}\left(d/n\right)$, being in $[d]$, can be interpreted as a \textit{hidden dimension} (relative to the missing pattern distribution). 
Indeed, the initial complexity scaling as $2^d$ in~\cref{thm:Risque_Complet} is replaced by 
$(\frac{ed}{s_{\epsilon}\left(d/n\right)})^{s_{\epsilon}\left(d/n\right)}$ 
in  \Cref{prop:BernoulliEBound} for this Bernoulli model.

Observe that the bound improves as $\epsilon $ decreases, for example for $\epsilon \leq \frac{d}{n}$, the excess risk bound scales as $\frac{d^2}{n}$. This again  highlights the benefit of adaptivity in \Cref{thm:borne_reg_threhold}, which allows us to obtain a bound that improves when the fraction of missing data decreases below a certain level.




 
 We extend the result above to the \textit{heterogeneous} case in \Cref{sec:hetero}, 
 and provide a discussion on the comparison between the complexities for  homogeneous and heterogeneous Bernoulli models that share \textit{the same} overall fraction of missing data~$\epsilon$ in \Cref{sec:bernoulli_num}.

\subsubsection{Example 3: Database Merge Model}\label{sec:database}

Consider a context of multi-sources data, where for instance a medical register results from merging $d$-dimensional data coming from $h$ different hospitals:
\begin{enumerate}[noitemsep,topsep=0pt]
\item each hospital $k\in\left\{ 1,\ldots,h\right\} $ has its own measurement
protocol, resulting in the missing pattern $P_{k}$ ($P_{k,j}=1$ if measure $j\in \{1,\hdots , d\}$,  is not performed in hospital $k$). Note that  
this missing pattern is shared by all the patients in care in hospital $k$.
\item in addition, for each measure $j\in\{1,\ldots,d\}$, the measuring device may make
a protocol-independent error, that produces a missing value with
probability $\eta$. 
\end{enumerate}
For an entry of the merged medical register, call $P$ (taking values in $P_1,\hdots, P_h$) the missing pattern  coding for the protocol effective in the hospital where this information has been collected, and $N \in \{0,1\}^d$ the missing pattern coding for the measurement failure.
Therefore, the eventual missing value pattern can be decomposed as,
\begin{equation}\label{eq:DatabaseMissing}
   (1-M)=(1-P_{H})\odot (1-N),  
\end{equation}
where $\odot$ is the Hadamard product.

This model is compatible with MNAR missing data mechanisms. 
Indeed, the missing pattern may be informative about the missing data values, as it encloses information about the hospital where the data is collected, and thereby may depend on a certain type of population distribution (geographical location, level of wealth...) frequenting the above hospital.~\cref{thm:borne_reg_threhold} can be adapted in such a setting as follows. 

\begin{propo}
\label{prop:datamerge}
Under Assumptions of~\cref{thm:Risque_Complet},
\begin{align*}
  \esp\left[\mathcal{E}\left(T_L\hat{f}^{(d/n)}\right)\right]\lesssim a_n\left(\frac{ed}{s_{\eta}\left(d/n\right)}\right)^{s_{\eta}\left(d/n\right)}h\frac{d}{n}+ A_{\F_b},\label{ineq:HospitalModel}  
\end{align*}
where $s_\eta$ is defined in \Cref{prop:BernoulliEBound}.
\end{propo}

See~\cref{proof:database} for the proof. The excess risk bound in~\Cref{prop:BernoulliEBound} encompasses a term similar to that of the Bernoulli case involving only the measurement failure probability $\eta$ here, whereas the number of protocols $h$ linearly intervenes.
To understand why this could be an advantage, consider two hospitals ($h=2$) in which only 50\% of the variables are systematically measured, and assume that the probability of measurement failure $\eta$ equals $0.01$.
The overall proportion $\epsilon$ of missing values in the merged dataset is therefore high, i.e.\ $\epsilon=1-0.99/2 \simeq 0.5$. 
Altogether, the bound  in  \Cref{prop:datamerge} (controlled via $s_{\eta}$) improves upon the one of  \Cref{prop:BernoulliEBound} (controlled via $s_{\epsilon}$)
by a factor $\epsilon/(h\eta)=25$.
This means that the bound in  \Cref{prop:BernoulliEBound} does not suffer from the resulting proportion $\epsilon$ of missing values, and mostly depends on the probability $\eta$ of measurement failure.  
This outlines the great plasticity of the complexity $\mathfrak{C}_p$ even in regimes with a large proportion of missing values, by leveraging the missing value structure.


\subsection{Minimax aspects}\label{sec:minimax}

In this section we discuss the optimality of the risk bound obtained for $\hat{f}^{(\tau)}$. To this end, we consider the class below.

\begin{definition}
\label{def:classPsigmaR}
The class of problems $\mathcal{P}_{p}(\sigma,R)$ is assumed to satisfy the following conditions: for all $\P\in\mathcal{P}_{p}(\sigma,R)$
\begin{enumerate}[noitemsep,topsep=0pt]
\item $\forall m\in\M,\P\left(M=m\right)=p_{m},$
\item $Y=\left\langle \beta,X\right\rangle +\epsilon$ where $\epsilon\sim\mathcal{N}\left(0,\sigma{{}^2}\right)$,
\item Assumptions \ref{ass:Lipsch} and \ref{ass:bornitude} hold with $B^2(\mm+1)\leq 3R^2$,
\item $A_{\F_b}=0$.
\end{enumerate}
\end{definition}
Note that this class of problems includes the Gaussian case (\cref{ass:Gaussian}) with M(C)AR, or GPMM (\cref{ass:ass4.1}). 
For this large class of problems, the excess risk can be upper bounded by \cref{thm:borne_reg_threhold} at the rate $a_n\mathfrak{C}_p(d/n)$, with $a_n\leq (\log(n)+1)^2 (R^2\vee \sigma_{\text{na}}^2)$. 
The following result provides a \textit{lower bound} on the excess risk with the same dependency on the complexity $\mathfrak{C}_p$.

\begin{theorem} \label{thm:minimaxNAR}
Consider a distribution $p$ on $\mathcal{M}$, then $R,\sigma>0$ and $c$ be such that $ 16e^{-\frac{1}{4}\left(\frac{R}{d\sigma}\right)^{2}} \leq c$. 
Therefore,
\begin{equation*}
\begin{split}
     (1-c)\sigma^{2}\mathfrak{C}_{p}\left(\frac{1}{n}\right)\lesssim \min_{\hat{f}}\max_{\mathbb{P}\in\mathcal{P}_{p}(\sigma, R)}\esp_{\P}\left[\mathcal{E}\left(\hat{f}\right)\right] \ .
\end{split}
\end{equation*}
where the minimum is over all predictor $\hat{f}$.
\end{theorem}

This result highlights the relevancy of the complexity $\mathfrak{C}_{p}$ in the control of the excess risk.
Since $\mathfrak{C}_{p}\left(\frac{1}{n}\right) \geq d^{-1}\mathfrak{C}_{p}\left(\frac{d}{n}\right)$, the lower bound in Theorem~\ref{thm:minimaxNAR} is sharp up to a factor~$d$. Note that if the distribution of missing patterns is uniform, one gets $\mathfrak{C}_{p}\left(\frac{1}{n}\right) =\frac{2^d}{n}$, meaning that 
the upper-bound of~\cref{thm:borne_reg_threhold} cannot be improved in full generality. 
Restricting the considered class to the MAR ones does not impact the lower-bound, as outlined in what follows.
\begin{corollary}\label{cor:minimaxMAR}
Assume that one component of $X$ is always observed. Then $\mathcal{P}_{p}(\sigma,R)\cap\mathcal{P}_{\text{MAR}}$ is non empty and \begin{equation*}
 (1-c)\sigma^{2}\mathfrak{C}_{p}\left(\frac{1}{n}\right)\lesssim \min_{\hat{f}}\max_{\mathbb{P}\in\mathcal{P}_{p}(\sigma,R)\cap\mathcal{P}_{\rm{MAR}}}\esp_{\P}\left[\mathcal{E}\left(\hat{f}\right)\right].
\end{equation*}
\end{corollary} 
This lower bound is of the same order as that of the upper bound in  \Cref{thm:minimaxNAR}.
While the MAR hypothesis facilitates the inference framework (the former actually originates from the latter, see~\citealt{RUBIN76}),  \Cref{cor:minimaxMAR} emphasizes that MAR scenarios do not help prediction purposes. 

\section{Numerical experiments}\label{sec:experiments}

In this section, we numerically evaluate the performance of several regressors on varying missing data scenarios.

\paragraph{Regressors.} More specifically, we compare the following five regression methods. First, we consider two baselines consisting in imputation followed by standard linear regression (on the completed data):  for \textbf{Cst-imp+LR} we learn  optimal imputation constants for each variable (note that this is equivalent to performing a LR of $Y$ on $\left(X_{\rm{obs}(m)},M\right)$, see \citep[Proposition 3.1]{le2020linear}); for \textbf{MICE+LR}, the imputation is performed by the scikit-learn \verb|IterativeImputer|
which relies on {MICE} \cite{van2011mice}.
Moreover, we add two pattern-by-pattern methods, that learn one regression model per pattern as defined in \Cref{eq:reg_with_tau}: for all patterns having at least one observation in \textbf{P-by-P imp} (i.e., $\tau=n^{-1}$ which matches the regressor in \eqref{eq:OLS_KD}), and 
with $\tau=d/n$ for \textbf{Thresholded P-by-P imp}. For both, the technical $\ell^\infty$-ball condition is not considered in numerical experiments. Finally
{\textbf{NeuMiss} \cite{lemorvan:hal-02888867} is a neural network which architecture is specifically designed to handle missing data in linear regression.} 


\paragraph{Data generation settings.} 
We consider three different settings in dimension $d=8$ with increasing difficulty:  
(a) \textbf{MCAR Bernoulli} in which $X$ and $M$ are independent, $M$ is generated according to the homogeneous Bernoulli Model of~\cref{sec:bernoulli} with missing value proportion $\epsilon=10\%$ and $X\sim\mathcal{N}(\mu,\Sigma)$ where $\mu\neq 0$ and $\Sigma\neq I$; 
(b) \textbf{MAR} in which $X$ is separated into two blocks of components $X^{(1)},X^{(2)}$ each of size $4$, $X^{(1)}$ is a Gaussian isotropic vector that is always observed and the missing pattern associated to $X^{(2)}$ is $M^{(2)}=\ind_{X^{(1)}>0}$, and $X^{(2)}|X^{(1)}\sim\mathcal{N}(M^{(2)},\Sigma)$ where $\Sigma\neq I$. 
(c) \textbf{MNAR-GPMM} in which $(X,M)$ is distributed according to~\cref{ass:ass4.1} with $7$ non-null probability missing patterns. See~\cref{sec:GMMdetails} for details.

\begin{figure*}[h]
    \centering
    \begin{tabular}{ccc}
        \includegraphics[width=0.3\linewidth]{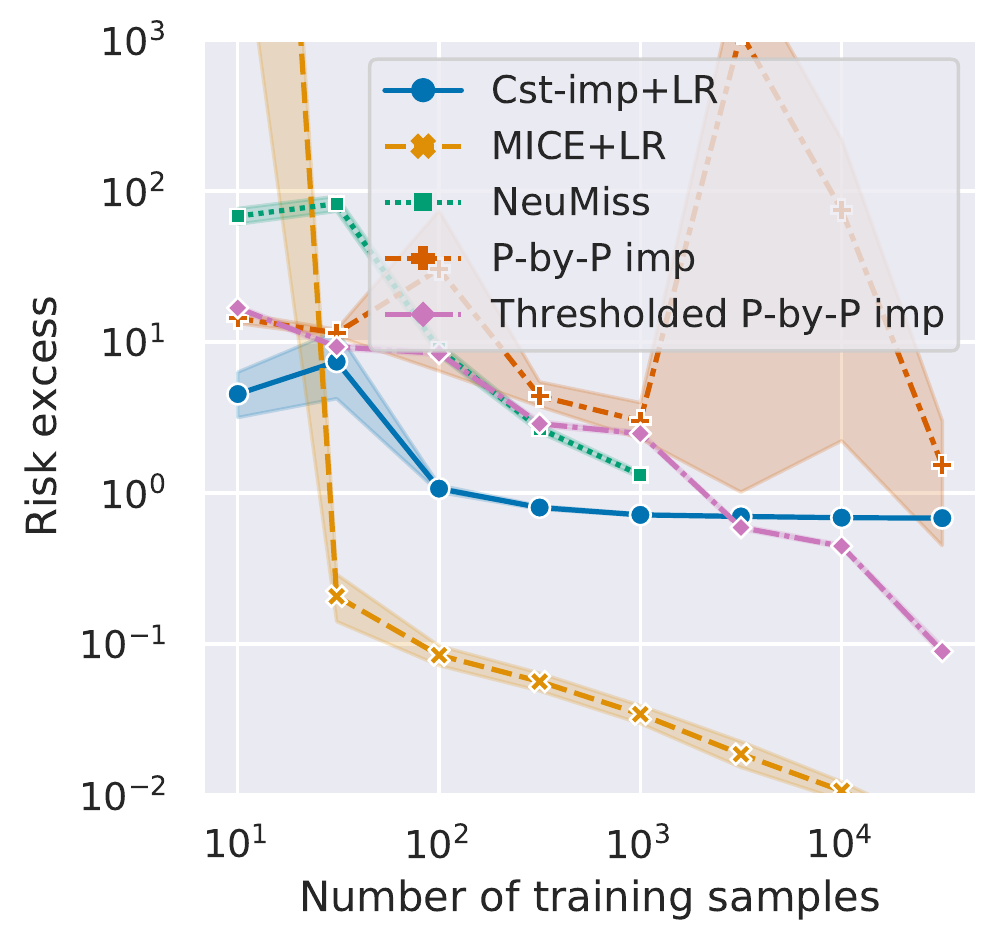}&
        \includegraphics[width=0.3\linewidth]{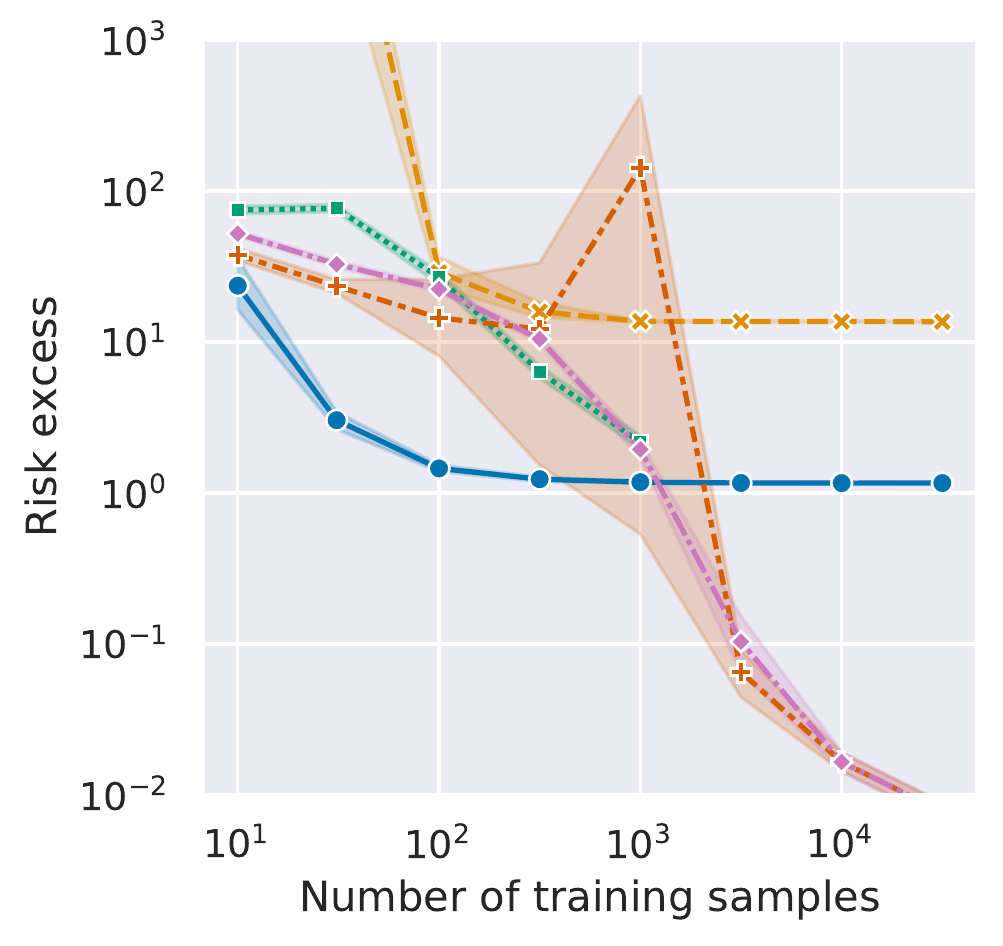}&
        \includegraphics[width=0.3\linewidth]{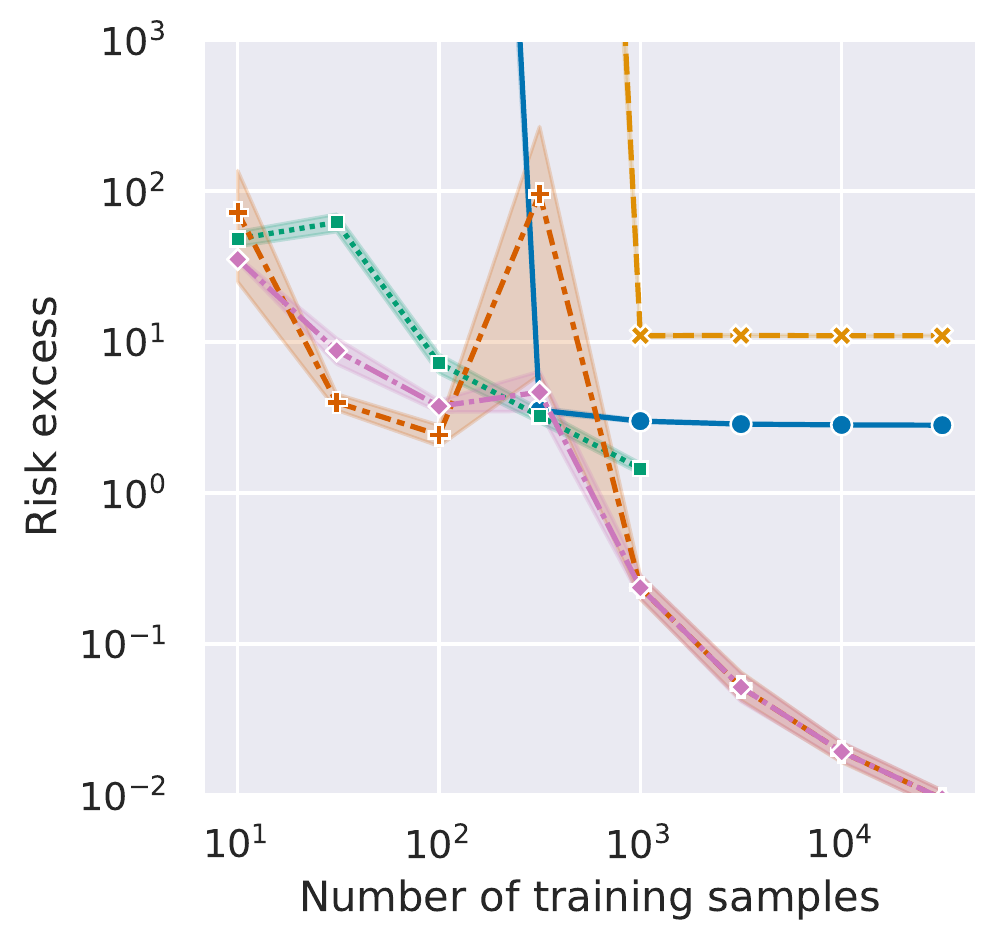} \\
        {(a) MCAR} & {(b) MAR} & {(c) MNAR-GPMM}
    \end{tabular}
    \caption{\label{fig:excess_risk_XP} Excess risk w.r.t.~the number of training samples. The curve represents the averaged excess risk over 100 repetitions within a 95\% confidence interval. }
\end{figure*}

\paragraph{Results.} The results are presented in Figure \ref{fig:excess_risk_XP}.
First, the P-by-P methods (with and without threshold) and NeuMiss are the only ones that are Bayes consistent regardless of the scenario (the excess risk tends to 0 on Figures \ref{fig:excess_risk_XP}(a,b,c)).
Neumiss provides similar performances at least in the MCAR and MAR settings, but its computational complexity, even in dimension $d=8$, prevents from reaching large sample sizes (see~\Cref{sec:learningTime}).
All the previous methods clearly outperform the MICE+LR strategy as soon as the data are not MCAR anymore, by exploiting the information contained in the missing pattern. Note that the Cst-imp method poorly performs whatever the data setting is: this could be explained by the fact that the model includes $2d$ parameters, which is not sufficient to learn the correlations between the variables (which would require $d^2$ parameters at least). 
Secondly, we remark the benefit of thresholding in P-by-P methods: \textbf{Thresholded P-by-P} outperforms the unthresholded version in particular for a small number of samples. 
Thresholding thus acts as a regularizer, by avoiding overfitting on the least frequent missing patterns.
\section{Conclusion}
In this paper, we propose a wide panel of data settings to study linear models with missing data. Contrary to most previous works, we focus on the prediction problem by evaluating the quadratic risk of linear models. We propose a new thresholded predictor coming with strong theoretical guarantees: the upper bound on its excess risk holds under very mild assumptions on the data, while integrating the complexity $\mathfrak{C}_p$, the missing patterns distribution. This quantity is interesting on its own as it describes the influence of the missing data distribution on the predictive performances. Several examples and a lower bound highlights the sharpness of our results. Numerical experiments emphasizes the improvement of our pattern-by-pattern estimator compared to state-of-the-art algorithms. 

{Training thresholded pattern-by-pattern predictors is a way to regularize the learning process highly complex when missing data occur.
Other types of regularization should be investigated to break the induced curse of dimensionality. However, the lower bound on the minimax predictor suggest that current assumptions are not strong enough to obtain better guarantees.
In the formalism of prediction with missing values, finding suitable assumptions on the missing patterns still remains an open question. 
}

\bibliography{references}

\appendix
\section{Proofs for Section \ref{sec:Generalization}}\label{proof:Gyorfi_freeBound}
\subsection{Key lemma on Binomial law}
Let's begin by a useful lemma on binomial law, which can be found in \citep[Lemma A2 p 587]{devroye2013probabilistic}.
\begin{lemma}\label{lem:inverse_bernoulli}
Let $B\sim\mathcal{B}(p,n)$, we have 
\begin{equation}
\frac{1}{1+np}\leq\esp\left[\frac{1}{1+B}\right]\leq\frac{1}{p(n+1)},\label{eq:inverse_binome1}
\end{equation}
and 
\begin{equation}
\esp\left[\frac{\ind\{B>0\}}{B}\right]\leq\frac{2}{p(n+1)}.\label{eq:inverse_binome2}
\end{equation}
\end{lemma}
\begin{proof}
 
\begin{itemize}
\item Lower bound of \eqref{eq:inverse_binome1}: we use Jensen inequality
\[
\frac{1}{1+np}=\frac{1}{1+\esp B}\leq\esp\left[\frac{1}{1+B}\right].
\]
\item Upper bound of \eqref{eq:inverse_binome1}, 
\begin{align*}
\esp\left[\frac{1}{1+B}\right] & =\sum_{i=0}^{n}\binom{n}{i}\frac{1}{1+i}p^{i}(1-p)^{n-i}\\
 & =\sum_{i=0}^{n}\frac{n!}{i!(n-i)!(1+i)}p^{i}(1-p)^{n-i}\\
 & =\frac{1}{\left(n+1\right)p}\sum_{i=0}^{n}\frac{(n+1)!}{(i+1)!(n+1-i-1)!}p^{i+1}(1-p)^{n-i}\\
 & =\frac{1}{\left(n+1\right)p}\sum_{i=0}^{n}\binom{n+1}{i+1}p^{i+1}(1-p)^{n+1-i-1}\\
 & \leq\frac{1}{\left(n+1\right)p},
\end{align*}
using binomial formula. 
\item For \eqref{eq:inverse_binome2}, we use that $1/x\leq2/(x+1)$ on $x\geq1$ and previous result. 
\end{itemize}
\end{proof}

\subsection{A key intermediate result on regression}
First, let us mention a very useful theorem for analyzing the quadratic risk in the regression framework. 
\begin{theorem}[Theorem 11.3 in \cite{gyorfi2006distribution}]\label{thm:gyorfi}
	Let two random variables $X$ and $Y$ be such that 
	\begin{equation*}
		Y= f^{\star}(X)+\epsilon
	\end{equation*}
	and 
	\begin{equation*}
		\begin{cases}
			\norm{f^{\star}}_{\infty}=\sup_{x\in\supp{X}}|f^{\star}(x)|\leq L\\
			\sigma^2=\sup_{x\in\supp{X}}\V{Y|X=x}<\infty
		\end{cases}
	\end{equation*}
for some $L>0$. Let $\F$ be a linear vector space of function $f:\R^d\longrightarrow\R$. Define the estimate by $T_Lf_n(x)=(-L)\vee f_n(x)\wedge L $ where 
\begin{equation}\label{eq:OLS}
	f_n\in\arg\min_{f\in\F}\frac{1}{n}\sum_{i=1}^{n}(f(X_i)-Y_i)^2.
\end{equation}
Then 
\begin{equation}\label{eq:risque_gyo}
\esp\left[\left(T_{L}f_{n}(X)-f^\star(X)\right)^{2}\right]\leq c\max\left\{ \sigma^{2},L^{2}\right\} \frac{\text{dim}(\F)(1+\log(n))}{n}+8\inf_{f\in\F}\esp\left[\left(f(X)-f^\star(X)\right)^{2}\right],
\end{equation}
for some universal constant $c$. 
\end{theorem}
The main drawback of this theorem is that it is only useful if the support of X is bounded.~\cref{ass:ass4.1} requires the covariates to be unbounded as they are assumed to be Gaussian. However, the covariates are on a bounded set with a high probability. The following corollary is adapted to this case.  
\begin{corollary}[Unbounded case]\label{cor:gyorfi}
	Let two random variable $X$, $Y$ and a subset $\K\subset\R^d$ be such that 
	\begin{equation*}
		\begin{cases}
			\norm{f^{\star}}_{\infty,\K}=\sup_{x\in\K}|f^{\star}(x)|\leq L_{\K}\\
			\sigma^2=\sup_{x\in\supp{X}}\V{Y|X=x}<\infty
		\end{cases}
	\end{equation*}
	for some $L_{\K}>0$. Let $\F$ be a linear vector space of function $f:\R^d\longrightarrow\R$. 
	Define 
	\begin{align}
	\label{def:f_K}
	    f_{\K}(x) :=T_{L}f_{n,\K}(x)\ind_{x\in\K},
	\end{align} where 
	\begin{equation}\label{eq:OLSechantillon}
		\begin{cases}
			f_{n,\K}\in\arg\min_{f\in\F}\sum_{X_i\in\K}(f(X_i)-Y_i)^2& \text{if } \exists i\in[n],X_i\in\K,\\
			f_{n,\K}=0& \text{else.}
		\end{cases}
	\end{equation}
	Then 
	\begin{equation}\label{eq:risque_echantillon}
		\begin{split}
			\esp\left[\left(f_{\K}(X)-f^\star(X)\right)^{2}\right]\leq c\max\left\{ \sigma^{2},L_{\K}^{2}\right\} \frac{\text{dim}(\F)(1+\log(n))}{n}\\+8\inf_{f\in\F}\esp\left[\ind_{X\in \K}\left(f(X)-f^\star(X)\right)^{2}\right] +R_{\K},
		\end{split}	
	\end{equation}
	where $p_{\K}=\prob{X\in \K}$ and $R_{\K}=(1-p_{\K})^n\e{\ind_{X\in \K}f^\star(X)^2}+\e{\ind_{X\notin \K}f^\star(X)^2}$.
\end{corollary}
Compared to~\cref{thm:gyorfi}, the bound obtained in the previous corollary includes an additional term $R_{\K}$. The extension of~\cref{thm:gyorfi} to the unbounded covariates case as done in Corollary \ref{cor:gyorfi} will be therefore informative only if this new term $R_{\K}$ remains of small order compared to the other ones. In the next corollary, we will apply it for sub Gaussian covariates.

In particular, under assumption $\esp\left[f^{*}\left(Z\right)^{4}\right]<+\infty$,
we can use Cauchy-Schwarz inequality to obtain 
\begin{equation}
R_{\K}\leq2(1-p_{K})^{1/2}\esp\left[f^{*}\left(Z\right)^{4}\right]^{1/2}.\label{ineq:BorneM4}
\end{equation}
 
\paragraph{Proof of~\cref{cor:gyorfi}}
The main idea of the proof is to consider the subsample of observations that are in $\K$: 
\begin{equation}\label{eq:defEK}
	E_{\K}=\left\lbrace i\in[n],X_i\in\K \right\rbrace .
\end{equation}

\textbf{Step 1: Law on subsample}\\
Let's start with a useful lemma to describe the elements of subsample induced by $E_{\K}$: 
\begin{lemma}
\label{lem:rejection_sampling}
Let $S=(X_i)_{i\in\N}$ be a sequence of independent variables with same distribution and $i_S=\inf\{i,X_i\in\K\}$. We suppose that $p_{\K}>0$, then $i_S<\infty$ almost surely and $ X_{i_S} $ has the same distribution as $X|(X\in\K)$.  
\end{lemma}
\begin{proof}
 $\prob{i_S> k}=p_{ K}^k$ thus $\sum \prob{i_S> k}$ is convergent. Borel-Cantelli lemma shows that $i_S<\infty$ almost surely. Consider a bounded function $\phi$:
\begin{align*}
	\e{\phi(X_{i_S})}&=\sum_{k=1}^{\infty}\prob{i_S= k}\e{\phi(X_i)|i_S=k}\\
	&=\sum_{k=1}^{\infty}\prob{i_S= k}\e{\phi(X_k)|X_k\in \K;X_1,...,X_{k-1}\notin\K}\\
	&=\sum_{k=1}^{\infty}\prob{i_S= k}\e{\phi(X_k)|X_k\in \K} \qquad  \text{because } X_k\text{ does not depend on }X_j,j<k\\
	&= \e{\phi(X_{1})|X_1\in \K}.	
\end{align*} 
This concludes the lemma.
\end{proof} 

Thanks to Lemma \ref{lem:rejection_sampling}, we can show $(X_i,Y_i)\sim \text{Law}\left( (X,Y)|(X\in\K)\right) $ for all $i\in E_{\K}$. \\Let $ (\widetilde{X},\widetilde{Y}) \sim \text{Law}\left( (X,Y)|X\in\K\right)$. Using the same notations of lemma, we can write the Bayes predictor for the regression problem involving the ``conditional" data $(\widetilde{Y},\widetilde{X})$: For all $x\in\K$,
\begin{align*}
	\widetilde{f}^{\star}(x)&=\e{\widetilde{Y}|\widetilde{X}=x}\\
	&=\e{Y_{i_S}|X_{i_S}=x}\\
	&=\sum_{k=1}^{\infty}\prob{i_S= k}\e{Y_{k}|X_k=x\in \K;X_1,...,X_{i-1}\notin\K}\\
	&=\sum_{k=1}^{\infty}\prob{i_S= k}\e{Y_{k}|X_k=x\in \K} \qquad  \text{because } X_k\text{ does not depend on }X_j,j<k\\
	&=\e{Y|X=x}\\
	&=f^{\star}(x). 
\end{align*}
Thus, 
\begin{equation}\label{eq:bayes_echantillon}
	\forall x\in \supp{X}, \quad\widetilde{f}^{\star}(x)=\ind_{x\in\K}f^{\star}(x).
\end{equation}
\textbf{Step 2: Decomposition of excess risk}\\
We can decompose:
\begin{align}
	\esp\left[\left(f_{\K}(X)-f^\star(X)\right)^{2}\right]&\leq\esp\left[\ind_{X\in\K}\left(f_{\K}(X)-f^\star(X)\right)^{2}\right]+\esp\left[\ind_{X\notin\K}\left(f_{\K}(X)-f^\star(X)\right)^{2}\right]\\
	&\leq p_{\K}\esp\left[\left(f_{\K}(X)-f^\star(X)\right)^{2}|X\in\K\right]+\esp\left[\ind_{X\notin\K} f^\star(X)^{2}\right]\label{eq:decompo_risque_echantillon},
\end{align}
using definition for the second term. We will bound the first term using Theorem \ref{thm:gyorfi} by conditioning according to $E_{\K}$. 
\begin{align}
	\esp\left[\left(f_{\K}(X)-f^\star(X)\right)^{2}|X\in\K\right]&=\esp\left[\left(f_{\K}(\widetilde{X})-\widetilde{f}^\star(\widetilde{X})\right)^{2}\right] & \text{by definition}\\
	&=\esp\left[\e{\left(f_{\K}(\widetilde{X})-\widetilde{f}^\star(\widetilde{X})\right)^{2}|E_{\K}}\right].
\end{align} 
Let $ E $ be a subset of $[n]$:
\begin{itemize}
	\item If $E$ is empty, \begin{equation}
		\e{\left(f_{\K}(\widetilde{X})-\widetilde{f}^\star(\widetilde{X})\right)^{2}|E_{\K}=E} =\e{\widetilde{f}^\star(\widetilde{X})^{2}}=\e{f^\star(\widetilde{X})^{2}},
	\end{equation}
using that $f_{\K}=0$ if $E_{\K}$ is empty. 
\item If $E$ is non-empty, $f_{\K}$ is the clipped OLS estimator for the problem $(\widetilde{Y},\widetilde{X})$ and the dataset $(Y_i,X_i)_{i\in E}$, thus
\begin{equation}
	\begin{split}
		\esp\left[\left(f_{\K}(\widetilde{X})-f^\star(\widetilde{X})\right)^{2}|E_{\K}=E\right]\leq c\max\left\{ \sigma^{2},L_{\K}^{2}\right\} \frac{\text{dim}(\F)(1+\log(|E|))}{|E|}\\+8\inf_{f\in\F}\esp\left[\left(f(\widetilde{X})-f^\star(\widetilde{X})\right)^{2}\right],
	\end{split}	
\end{equation}
using Theorem \ref{thm:gyorfi}. 
\end{itemize}
Therefore, 
\begin{equation}
	\begin{split}
		\esp\left[\left(f_{\K}(\widetilde{X})-f^\star(\widetilde{X})\right)^{2}|E_{\K}\right]\leq c\max\left\{ \sigma^{2},L_{\K}^{2}\right\} \frac{\text{dim}(\F)(1+\log(n))}{|E_{\K}|}\ind_{|E_{\K}|>0}+\ind_{|E_{\K}|=0}\e{f^\star(\widetilde{X})^{2}}\\+8\inf_{f\in\F}\esp\left[\left(f(\widetilde{X})-f^\star(\widetilde{X})\right)^{2}\right],
	\end{split}	
\end{equation}
using $\log(|E_{\K}|)\leq \log(n)$. Moreover $|E_{\K}|$ is a binomial distribution with parameters $p_{\K}$ and $n$, thus 

\begin{equation}\label{eq:Gyo_condi}
	\begin{split}
		\esp\left[\left(f_{\K}(\widetilde{X})-f^\star(\widetilde{X})\right)^{2}\right]\leq c\max\left\{ \sigma^{2},L_{\K}^{2}\right\} \frac{2\text{dim}(\F)(1+\log(n))}{(n+1)p_{\K}}+(1-p_{\K})^n\e{f^\star(\widetilde{X})^{2}}\\+8\inf_{f\in\F}\esp\left[\left(f(\widetilde{X})-f^\star(\widetilde{X})\right)^{2}\right], 
	\end{split}	
\end{equation}
using the expectation of the inverse of a binomial distribution. By combining  (\ref{eq:Gyo_condi}) and (\ref{eq:decompo_risque_echantillon}), and using 
\begin{align*}
	\esp\left[\left(f(\widetilde{X})-f^\star(\widetilde{X})\right)^{2}\right]&=\esp\left[\left(f(X)-f^\star(X)\right)^{2}|X\in\K\right]\\
	&=p_{\K}^{-1}\esp\left[\ind_{X\in \K}\left(f(X)-f^\star(X)\right)^{2}\right],
\end{align*}
we find
\begin{equation}
	\begin{split}
		\esp\left[\left(f_{\K}(X)-f^\star(X)\right)^{2}\right]\leq c\max\left\{ \sigma^{2},L_{\K}^{2}\right\} \frac{\text{dim}(\F)(1+\log(n))}{n}\\+8\inf_{f\in\F}\esp\left[\ind_{X\in \K}\left(f(X)-f^\star(X)\right)^{2}\right] +R_{\K},
	\end{split}	
\end{equation}
where $p_{\K}=\prob{X\in \K}$ and $R_{\K}=(1-p_{\K})^n\e{\ind_{X\in \K}f^\star(X)^2}+\e{\ind_{X\notin \K}f^\star(X)^2}$. 

\begin{lemma}\label{lem:p_k}
 	Under ~\cref{ass:bornitude}, we have for all $D>\sqrt{\mm}$,
 	\begin{equation}\label{eq:highPX_i}
 		\forall i\in[p], \qquad \forall D>\sqrt{\mm},\qquad\P\left[|X_{i}|>D\right]\leq2\exp\left[-\frac{\left(D-\sqrt{\mm}\right)^{2}}{2\mm}\right],
 	\end{equation}
 	and
 	\begin{equation}\label{eq:probaKD}
 		\prob{Z\notin\K_D}\leq 2d\exp\left(-\frac{\left(D-\sqrt{\mm}\right)^{2}}{2\mm}\right).
 	\end{equation}
\end{lemma}
\begin{proof}
	Let's fix $m\in\M$ and $i\in[p]$, by Assumption \ref{ass:bornitude}:
	By bounding the tail of a Sub-Gaussian distribution, 
	\begin{equation}
	    \prob{\left|X_i-\esp X_i\right|>D}\leq 2\exp\left(-\frac{D^2}{2\gamma}\right).
	\end{equation}
	Since $\left|X_{i}\right|+\sqrt{\mm}\geq \left|X_{i}\right|+\left|\esp X_i\right|\geq\left|X_{i}-\esp X_i\right|$, we have 
	\begin{equation*}
 		\forall i\in[p],\forall D>\sqrt{\mm},\qquad\P\left(|X_{i}|>D\right)\leq2\exp\left(-\frac{\left(D-\sqrt{\mm}\right)^{2}}{2\mm}\right). 
 	\end{equation*}
 	For the second point, we make a union bound and we use (\ref{eq:highPX_i}):
 \begin{align*}
\P\left(Z\notin \K_{D}\right)= &  \P\left(\left\Vert X_{\rm{obs}(m)}\right\Vert_{\infty}>D\right)\\
 & \leq\P\left(\left\Vert X\right\Vert _{\infty}>D\right)\\
 & \leq\sum_{i=1}^{d}\P\left(\left|X_{i}\right|>D\right)\\
 & \leq2d\exp\left(-\frac{\left(D-\sqrt{\mm}\right)^{2}}{2\mm}\right). 
\end{align*}
\end{proof}

\subsection{Proof of Theorem \ref{thm:Risque_Complet}}\label{proof:Risque_Complet}
The proof consists in verifying the assumptions of the corollary \ref{cor:gyorfi} and to bound the additional term $R_{\K_D}$.
\begin{itemize}
\item $\sup_{Z\in\K_{D}}\left|f^\star(Z)\right|$ upper bound:
\begin{align*}
\sup_{Z\in\K_{D}}\left|f^\star(Z)\right| & =\sup_{m\in\M}\sup_{\left\Vert x_{\rm{obs}(m)}\right\Vert _{\infty}\leq D}\left|f^\star_m(x_{\rm{obs}(m)}) \right|\\
& =\sup_{m\in\M}\sup_{\left\Vert x_{\rm{obs}(m)}\right\Vert _{\infty}\leq D}\left|f^\star_m(x_{\rm{obs}(m)}) -f^\star_m(0)\right|+\left|f^\star_m(0)\right|\\
\\
& =\sup_{m\in\M}\sup_{\left\Vert x_{\rm{obs}(m)}\right\Vert _{\infty}\leq D}B\left\Vert x_{\rm{obs}(m)}\right\Vert _{\infty}+B,
\end{align*}
using Assumption \ref{ass:Lipsch}, thus
\begin{equation}\label{eq:def_L}
  \sup_{Z\in\K_{D}}\left|f^\star(Z)\right|\leq (D+1)B\leq L.  
\end{equation}

\item $R_{\K_D}$ upper bound: From \eqref{ineq:BorneM4} 
\begin{align}
R_{\K} & \leq2(1-p_{\K_D})^{1/2}\esp\left[f^{*}\left(Z\right)^{4}\right]^{1/2}\nonumber \\
 & \leq2\sqrt{\P\left(Z\notin\K_D\right)\esp\left[f^{*}\left(Z\right)^{4}\right]}\nonumber \\
 & \leq2\sqrt{2d}\exp\left(-\frac{\left(D-\sqrt{\mm}\right)^{2}}{4\mm}\right)\sqrt{\esp\left[f^{*}\left(Z\right)^{4}\right]}.\label{eq:borne_Rk1}
\end{align}
It remains to bound
\begin{align}
\esp\left[f^{*}\left(Z\right)^{4}\right] & =\sum_{m\in\M}\P\left(M=m\right)\esp\left[f_{m}^{*}\left(X_{\rm{obs}(m)}\right)^{4}|M=m\right]\nonumber \\
 & =\sum_{m\in\M}\P\left(M=m\right)\esp\left[\left|\left|f_{m}^{*}\left(X_{\rm{obs}(m)}\right)-f_{m}^{*}\left(0\right)\right|+\left|f_{m}^{*}\left(0\right)\right|\right|^{4}|M=m\right]\nonumber \\
 & \leq8\sum_{m\in\M}\P\left(M=m\right)\esp\left[B{{}^4}\left\Vert X_{\rm{obs}(m)}\right\Vert _{2}^{4}+B^{4}|M=m\right] & \text{by Assumption \ref{ass:Lipsch}}\nonumber \\
 & \leq8B^{4}\sum_{m\in\M}\P\left(M=m\right)\esp\left[\left\Vert X\right\Vert _{\infty}^{4}|M=m\right]+B^{4}.\nonumber \\
 & \leq8B^{4}\left(\esp\left[\left\Vert X\right\Vert _{\infty}^{4}\right]+1\right).\label{eq:borne_Rk2}
\end{align}
Then,
\begin{align}
\esp\left[\left\Vert X\right\Vert _{\infty}^{4}\right] & \leq\esp\left[\left(\left\Vert X-\esp X\right\Vert _{\infty}+\left\Vert \esp X\right\Vert _{\infty}\right)^{4}\right] & \text{by triangular inequality }\nonumber \\
 & \leq8\esp\left[\left\Vert X-\esp X\right\Vert _{\infty}^{4}+\left\Vert \esp X\right\Vert _{\infty}^{4}\right]\nonumber \\
 & \leq8\esp\left[\left\Vert X-\esp X\right\Vert _{\infty}^{4}\right]+8\gamma^{2} & \text{by Assumption \ref{ass:bornitude}}\nonumber \\
 & \leq8\esp\left[\max_{j\in[d]}\left|X_{j}-\esp X_{j}\right|^{4}\right]+8\gamma^{2}\nonumber \\
 & \leq8\sum_{j=1}^{d}\esp\left[\left|X_{j}-\esp X_{j}\right|^{4}\right]+8\gamma^{2}\nonumber \\
 & \leq8d\left(2!\left(4\mm\right)^{2}\right)+8\gamma^{2}\label{eq:use_charec}\\
 & \leq257d\gamma^{2}\label{eq:borne_Rk3-1}.
\end{align}
We have used moment's characterization \citep[Theorem 2.2]{boucheron2013concentration} in \eqref{eq:use_charec}. Using \eqref{eq:borne_Rk1}, \eqref{eq:borne_Rk2} and \eqref{eq:borne_Rk3-1},
we have 
\[
R_{\K_{D}}\leq64B^{2}d^{1/2}\left(\mm+1\right)\exp\left(-\frac{\left(D-\sqrt{\mm}\right)^{2}}{4\mm}\right).
\]
The choice $D=\sqrt{\gamma}+\sqrt{4\mm\log(n)}$ leads to:

\begin{equation}\label{eq:R_KD_bound}
  R_{\K_{D}}  \leq64B^{2}\left(\mm+1\right)\frac{d}{n}.  
\end{equation}

\end{itemize}

Since Assumptions of Corollary \ref{cor:gyorfi}  are satisfied, we have

\begin{align*}\esp\left[\left(T_{L}\hat{f}^{(D\textrm{-}\rm{LS})}(Z)-f^{*}(Z)\right)^{2}\right] & \leq c\left[\max\left[\sigma_{\text{na}}^{2},B^{2}(D+1)^{2}\right]\frac{2^{d}d(\log(n)+1)}{n}+64B^{2}\left(\mm+1\right)\frac{d}{n}\right] +8  A_{\F_b}\\
 & \leq c_{2}\max\left(\sigma_{\text{na}}^{2},L^{2}\right)\frac{2^{d}d\left(\log(n)+1\right)}{n}+8  A_{\F_b}.
\end{align*}

\subsection{Proof of Corollary \ref{cor:freeBoundAndLin}}\label{proof:freeBoundAndLin}
We just have to check the assumptions of the Theorem \ref{thm:Risque_Complet}. 
\begin{itemize}
    \item Under Assumptions of Proposition \ref{prop:f_mLineaire} $f_m^\star$ is linear:
    \begin{align*}
\left|f_{m}(x)-f_{m}(y)\right| & \leq\left|\left\langle x-y,\delta^{(m)}\right\rangle \right|\\
 & =\left\Vert x-y\right\Vert _{\infty}\left|\left\langle \frac{x-y}{\left\Vert x-y\right\Vert _{\infty}},\delta^{(m)}\right\rangle \right|\\
 & \leq\left\Vert x-y\right\Vert _{\infty}\sup_{\left\Vert u\right\Vert _{\infty}\leq1}\left|\left\langle u,\delta^{(m)}\right\rangle \right|\\
 & =\left\Vert x-y\right\Vert _{\infty}\left\Vert \delta^{(m)}\right\Vert _{1},
\end{align*}
thus \ref{ass:Lipsch} is verified with $B=  \max_{m\in\M}\max\left[|\delta_{0}^{(m)}|,\left\Vert \delta^{(m)}\right\Vert _{1}\right]$. 
    \item Let's check Assumption \ref{ass:bornitude}:
\begin{itemize}
    \item Under Assumption \ref{ass:Gaussian} and (\ref{ass:MAR} or \ref{ass:MCAR}), for $j\in[d]$, $X_j$ is Gaussian, thus $X_j-\esp{X_j}$ is $\mathbb{V}[X_j]$-Sub-Gaussian and 
    \[
\begin{cases}
\esp\left[X_{j}\right]^{2}\leq\gamma\\
\mathbb{V}\left[X_{j}\right]\leq\gamma,
\end{cases}
\]
where $\gamma=\max_{j\in[d]}\esp\left[X_{j}^{2}\right]$. 
\item Under Assumption \ref{ass:ass4.1}: Let $\mm$ such that for all $m\in\M$ and $j\in[d]$, 
\begin{equation}
\esp\left[X_{j}^{2}|M=m\right]\leq\gamma\label{eq:corpro}
\end{equation}

We will use moment characterisation \citep[Theorem 2.1]{boucheron2013concentration} to proove  that the tail of $X_{j}$
is Sub-Gausian. 
\begin{align*}
\esp\left[\left(X_{j}-\esp X_{j}\right)^{2q}\right] & \leq\esp\left[X_{j}^{2q}\right]\\
 & =\sum_{m\in\M}p_{m}\esp\left[X_{j}^{2q}|M=m\right]
\end{align*}

By assumption, $X_{j}|M=m$ is Gaussian, we denote by $e_{j,m}$ is
expectancy, we have $e_{j,m}^{2}\leq\gamma$, from this it follow
that
\begin{align*}
\esp\left[X_{j}^{2q}|M=m\right] & =\sum_{k=0}^{2q}\binom{2q}{k}\esp\left[\left(X_{j}-e_{j,m}\right)^{k}|M=m\right]e_{j,m}^{2q-k}\\
 & =\sum_{k=0}^{q}\binom{2q}{2k}\esp\left[\left(X_{j}-e_{j,m}\right)^{2k}|M=m\right]e_{j,m}^{2q-2k}\\
 & \leq\sum_{k=0}^{q}\binom{2q}{2k}\esp\left[\left(X_{j}-e_{j,m}\right)^{2k}|M=m\right]\mm^{2q-2k}\\
 & \leq\sum_{k=0}^{q}\binom{2q}{2k}\frac{(2k)!}{2^{k}k!}\mm^{2k}\mm^{2q-2k}\\
 & \leq\mm^{2q}(2q)!\sum_{k=0}^{q}\frac{1}{2^{k}}\\
 & \leq2\mm^{2q}(2q)!.
\end{align*}

Thus, 
\[
\esp\left[\left(X_{j}-\esp X_{j}\right)^{2q}\right]\leq\left(2\mm\right)^{2q}(2q)!.
\]
This conclude that $X_{j}-\esp X_{j}$ is $8\mm$-Sub-Gaussian.
\end{itemize}
\end{itemize}

\section{Proofs of the main result from \cref{sec:main_result_regression}}

\subsection{Intermediate results}
The following lemma allows to control a key quantity, that appears when we try to separate the frequent patterns among those that are less frequent.  
\begin{lemma}\label{lem:R_tau_d}
Let's define 
\[
R_{\tau,p}(n):=\esp\left[\sum_{m\in\M}p_{m}\left(\frac{\ind_{|E_{m}|>\tau}}{\left|E_{m}\right|}d+\ind_{|E_{m}|=0}\right)\right],
\]where $E_{m}=\left\{ i\in[n],M_{i}=m\right\} $. We have for $\tau\geq 1/n$:
\begin{equation}
    R_{\tau,p}(n)\leq 5\max\left(1,\frac{d}{n\tau}\right)\mathfrak{C}_p(\tau).
\end{equation}
\end{lemma}
\begin{proof}
 \begin{itemize}
    \item Case 1: if $\tau>1$ then $R_{\tau,p}\left(\tau\right)=1$ and $\mathfrak{C}_{p}\left(\tau\right)=1$
thus $$R_{\tau,p}(n)\leq5 \mathfrak{C}_{p}\left(\tau\right).$$
    \item Case 2: We suppose that $\tau\leq1$, we have
$\frac{d}{n\tau}=1$. We denote by $K_\tau$ the cardinal number of $\{m\in\M,p_{m}>\tau\}$,   

\begin{align*}
R_{\tau,p}(n) & \leq\esp\left[\sum_{m\in\M}p_{m}\left(\frac{\ind_{|E_{m}|>n\tau}}{|E_{m}|}d+\ind_{|E_{m}|<n\tau}\right)\right]\\
 & =\esp\left[\sum_{m:p_{m}>\tau}p_{m}\left(\frac{\ind_{|E_{m}|>n\tau}}{|E_{m}|}d+\ind_{|E_{m}|<n\tau}\right)\right]+\esp\left[\sum_{m:p_{m}\leq\tau}p_{m}\left(\frac{\ind_{|E_{m}|>n\tau}}{|E_{m}|}d+\ind_{|E_{m}|<n\tau}\right)\right].
\end{align*}
Using that $\frac{d}{n\tau}=1$, we have $\frac{\ind_{|E_{m}|>n\tau}}{|E_{m}|}d+\ind_{|E_{m}|<n\tau}\leq \max\left(1,\frac{d}{n\tau}\right)$. Thus, 

\begin{align}
R_{\tau,p}(n) & \leq\esp\left[\sum_{m:p_{m}>\tau}p_{m}\left(\frac{\ind_{|E_{m}|>n\tau}}{|E_{m}|}d+\ind_{|E_{m}|\leq n\tau}\right)\right]+\max\left(1,\frac{d}{n\tau}\right)\sum_{m:p_{m}\leq\tau}p_{m}\nonumber \\
 & \leq\esp\left[\sum_{m:p_{m}>\tau}p_{m}\frac{\ind_{|E_{m}|>n\tau}}{|E_{m}|}d\right]+\sum_{m:p_{m}>\tau}p_{m}\P\left(|E_{m}|\leq n\tau\right)+\max\left(1,\frac{d}{n\tau}\right)\sum_{m:p_{m}\leq\tau}p_{m}.\label{eq:preuve_dn_1}
\end{align}
\begin{itemize}
    \item First term: We use ~\cref{lem:inverse_bernoulli}, $\esp\left[\frac{\ind_{|E_{m}|>1}}{|E_{m}|}\right]\leq\frac{2}{p_{m}(n+1)}$
because $\left|E_{m}\right|\sim\mathcal{B}(n,p_{m}).$
\begin{equation}
\esp\left[\sum_{m:p_{m}>\tau}p_{m}\frac{\ind_{|E_{m}|>n\tau}}{|E_{m}|}d\right]\leq2K_{\tau}\frac{d}{n}\leq2K_{\tau}\tau.\label{eq:preuve_dn_2}
\end{equation}
\item Second term: 
\begin{align}
\sum_{m:p_{m}>\tau}p_{m}\P\left(|E_{m}|\leq n\tau\right) & \leq\sum_{m:p_{m}>\tau}p_{m}\left(\P\left(\left|E_{m}\right|=0\right)+\P\left(\frac{\ind_{\left|E_{m}\right|>0}}{\left|E_{m}\right|}\geq1/\tau n\right)\right)\nonumber \\
 & \leq\sum_{m:p_{m}>\tau}p_{m}\left((1-p_{m})^{n}+\tau n\esp\left[\frac{\ind_{\left|E_{m}\right|>0}}{\left|E_{m}\right|}\right]\right) \qquad  \text{(Markov)}\nonumber \\
 & \leq\sum_{m:p_{m}>\tau}p_{m}(1-p_{m})^{n}+\sum_{m:p_{m}>\tau}p_{m}\tau n\frac{2}{p_{m}(n+1)}\nonumber \\
 & \leq\frac{K_{\tau}}{n}+2\tau K_{\tau}. \qquad \qquad \quad  \text{by optimizing }\nonumber \\
 & \leq3\tau K_{\tau} \qquad \qquad \qquad \qquad \text{because } 1/n\leq\tau.
 \label{eq:preuve_dn_3}
\end{align}
\end{itemize}

Combining (\ref{eq:preuve_dn_1}),(\ref{eq:preuve_dn_2}) and (\ref{eq:preuve_dn_3}),
we find 
\[
R_{\tau,p}(n)\leq5\tau K_{\tau}+\max\left(1,\frac{d}{n\tau}\right)\sum_{m:p_{m}\leq\tau}p_{m}\leq5\max\left(1,\frac{d}{n\tau}\right)\mathfrak{C}_{p}(\tau).
\]
\end{itemize}
\end{proof}

The next lemma is particularly useful to show that the optimal threshold is $\tau=d/n$.
\begin{lemma}\label{lem:tau_opti}
For all $\tau>0$,
\begin{equation*}
    \mathfrak{C}_{p}\left(\frac{d}{n}\right)\leq \max\left(1,\frac{d}{n\tau}\right)\mathfrak{C}_{p}\left(\tau\right).
\end{equation*}
\end{lemma}
\begin{proof}
Remark that for any $\tau$, 
\begin{align*}\label{eq:comparison_tau}
\mathfrak{C}_{p}\left(\frac{d}{n}\right)
&=\sum_{m\in\M}p_m\wedge \frac{d}{n}
=\sum_{m\in\M}p_m\wedge \left(\frac{d}{n\tau} \tau\right) \\
&\leq \sum_{m\in\M}p_m\wedge\left( \max\left(1,\frac{d}{n\tau}\right) \tau \right)
\\
&\leq \max\left(1,\frac{d}{n\tau}\right)\mathfrak{C}_{p}\left(\tau\right).
\end{align*}

Using that $\max\left(1,\frac{d}{n\tau}\right)\geq 1$.
\end{proof}

\subsection{Proof of~\cref{thm:borne_reg_threhold}}\label{proof_of_reg_tau}

\begin{proof}
In this proof we use same notations and some results of the proof of~\cref{thm:Risque_Complet}. 
We consider on each $m\in\M$, $\K_{m,D}:=\left\{ x_{\rm{obs}(m)},\left\Vert x_{\rm{obs}(m)}\right\Vert _{\infty}\leq D\right\} $. From \eqref{eq:def_L}, we have for $L=(D+1)B$
\begin{equation}
    \forall x\in\K_{m,D},\qquad |f_m^\star(x)|\leq L.
\end{equation}

Let's begin by a decomposition of excess risk. 
\begin{align*}
\esp\left[\left(T_{L}\widehat{f}^{(\tau)}(Z)-f^{\star}(Z)\right)^{2}\right] & =\sum_{m\in\M}p_{m}\esp\left[\left(T_{L}\widehat{f}^{(\tau)}(Z)-f^{\star}(Z)\right)^{2}|M=m\right]\\
 & =\sum_{m\in\M}p_{m}\esp\left[\left(T_{L}\widehat{f}_{m}^{(\tau)}(X_{\rm{obs}(m)})-f_{m}^{\star}(X_{\rm{obs}(m)})\right)^{2}|M=m\right],
\end{align*}
where $f_{m}^{\star}(X_{\rm{obs}(m)}):=\widetilde{f}_{m}\left(X_{\rm{obs}(m)}\right)\ind_{\widehat{p}_{m}>\tau}$.
Using~\cref{cor:gyorfi} on each $m$ with $$\K=\K_{m,D}=\left\{ x_{\rm{obs}(m)},\left\Vert x_{\rm{obs}(m)}\right\Vert _{\infty}\leq D\right\},$$ we have

\begin{equation*}
    \begin{split}
        \esp\left[\left(T_{L}\widehat{f}_{m}^{(\tau)}(X_{\rm{obs}(m)})-f_{m}^{\star}(X_{\rm{obs}(m)})\right)^{2}|M=m,E_{m}\right]\leq\ind_{|E_{m}|\leq\tau n}\esp\left[f_{m}^{\star}(X_{\rm{obs}(m)})^{2}|M=m\right]\\+c\sigma_{\text{na}}^{2}\vee L^{2}\frac{d}{|E_{m}|}\ind_{|E_{m}|>\tau n}\\+8\text{Approx}\left(f_{m}^{\star},\F_{m}\right)+R_{\K_{m,D}}.
        \end{split}
\end{equation*}

We split the first term: 
\begin{equation}
\begin{split}
    \ind_{|E_{m}|\leq\tau n}\esp\left[f_{m}^{\star}(X_{\rm{obs}(m)})^{2}|M=m\right]  =\ind_{|E_{m}|\leq\tau n}\esp\left[\ind_{\left\Vert x_{\rm{obs}(m)}\right\Vert _{\infty}\leq D}f_{m}^{\star}(X_{\rm{obs}(m)})^{2}|M=m\right]\\+\ind_{|E_{m}|\leq\tau n}\esp\left[\ind_{\left\Vert x_{\rm{obs}(m)}\right\Vert _{\infty}>D}f_{m}^{\star}(X_{\rm{obs}(m)})^{2}|M=m\right].
    \end{split}
\end{equation}
If $X_{\rm{obs}(m)}\in \K_{m,D}$ then $f_{m}^{\star}(X_{\rm{obs}(m)})^{2}\leq L^2$ and the second term is smaller than $R_{ \K_{m,D}}$. Thus, 
\begin{equation}
    \ind_{|E_{m}|\leq\tau n}\esp\left[f_{m}^{\star}(X_{\rm{obs}(m)})^{2}|M=m\right]\leq\ind_{|E_{m}|\leq\tau n}L^{2}+R_{\K_{m,D}}.
\end{equation}

By combining, 

\begin{equation*}
    \begin{split}
        \esp\left[\left(T_{L}\widehat{f}_{m}^{(\tau)}(X_{\rm{obs}(m)})-f_{m}^{\star}(X_{\rm{obs}(m)})\right)^{2}|M=m,E_{m}\right]\leq c\sigma_{\text{na}}^{2}\vee L^{2}\left(\frac{d}{|E_{m}|}\ind_{|E_{m}|>\tau n}+\ind_{|E_{m}|\leq\tau n}\right)\\+8\text{Approx}\left(f_{m}^{\star},\F_{m}\right)+2R_{\K_{m,D}}.
    \end{split}
\end{equation*}

By summing and taking expectation, we obtain
\begin{equation}
\begin{split}
\esp\left[\left(T_{L}\widehat{f}^{(\tau)}(Z)-f^{\star}(Z)\right)^{2}\right]=c\sigma_{\text{na}}^{2}\vee L^{2}\sum_{m\in\M}p_{m}\esp\left[\left(\frac{d}{|E_{m}|}\ind_{|E_{m}|>\tau n}+\ind_{|E_{m}|\leq\tau n}\right)\right]\\+8 A_{\F_b}+2R_{\K_{D}}.\label{eq:reg_tau_proove_decompo}
\end{split}
\end{equation}
We have used that $\sum_{m\in\M}p_{m}\text{Approx}\left(f_{m}^{\star},\F_{m}\right)= A_{\F_b}$
and $\sum_{m\in\M}p_{m}R_{\K_{m,D}}=R_{\K_{D}}.$ From Lemma \ref{lem:R_tau_d} we have 
\begin{equation*}
    \sum_{m\in\M}p_{m}\esp\left[\left(\frac{d}{|E_{m}|}\ind_{|E_{m}|>\tau n}+\ind_{|E_{m}|\leq\tau n}\right)\right] \leq 5\max\left(1,\frac{d}{n\tau}\right)\mathfrak{C}_p(\tau).
\end{equation*}
We recall that we
have from \eqref{eq:R_KD_bound} and \eqref{ineq:Evsconstante}:
\begin{equation}
R_{\K_{D}}\leq64L^{2}\frac{d}{n}\leq64L^{2}\mathfrak{C}_p(d/n)\leq 64L^{2} \max\left(1,\frac{d}{n\tau}\right)\mathfrak{C}_p(\tau),\label{eq:bound_RK_tau}
\end{equation}
using \cref{lem:tau_opti}. This concludes on \eqref{eq:reg_bound_tau1}.

The optimal choice of $\tau$ to minimize the upper bound \eqref{eq:reg_bound_tau1} is $\tau=d/n$, by a direct application of \cref{lem:tau_opti}. 
\end{proof}

\section{ Properties of $\mathfrak{C}_p$ and examples}\label{app:complexity}

\subsection{Insight on $\mathfrak{C}_{p}$}
\label{subsec:insight_complexity}
In this section, we will enumerate a number of results on $\mathfrak{C}_{p}$. In particular, thanks to the link with the notion of entropy, and the properties linking structure and complexity of distribution $p$, we can deal with examples such as the homogeneous and heterogeneous Bernoulli Model.  
\subsubsection{Link with entropies}

Computing $\mathfrak{C}_{p}\left(\frac{d}{n}\right)$  explicitly can be tricky and requires the knowledge of the distribution $p$ of the missing data patterns. The purpose of the following development is to control this complexity with generic bounds. 

\begin{definition}
Let $b>0$, let $\mathcal{P}_b(\M)$ be the set of $p\in\mathcal{P}(\M)$ such that for all $m\in\M$, $p_{m}\leq b$.
We define $\mathcal{G}_b$ the set of  function $g:(1/b,+\infty)\longrightarrow\R_{+}^\star$ such
that 
\begin{align*}
\text{\ensuremath{\left(G_{1}\right)}}:\qquad   & x\longmapsto g(x) &  \text{\ensuremath{\text{is non decreasing}}}\\
\left(G_{2}\right): \qquad   & x\longmapsto xg(1/x) &  \text{is non decreasing}.
\end{align*}
And, for all $p\in\mathcal{P}_b(\M)$, set
\[
H_{g}(p):=\sum_{m\in\M}p_{m}g(1/p_{m}).
\]
\end{definition}

Depending on the choice of $g$, the quantities $H_{g}(p)$ can convey some characteristics of the distribution $p$. For example, if $g=\text{id}$, $H_{g}(p)$ falls down to the cardinal of the support.  If now $g=\log$, this leads to the standard Shannon entropy.  Note that if we rewrite \eqref{def:E} as 
\begin{equation}\label{def:E_inv}
\forall\tau\in(0,1),\quad\mathfrak{C}_{p}\left(\tau\right):=\sum_{m\in\M}p_{m}\min\left(1,\frac{\tau}{p_m}\right),
\end{equation}
then,  $\mathfrak{C}_{p}(\tau)=H_{g}(p)$ for $g(x)=\min(1,\tau x)$. This gives the intuition of the following result. 

\begin{theorem}\label{thm:inf_form}
Let $b>0$, for all $p\in\mathcal{P}_b(\M)$ and $\tau\in(0,b),$ 
\begin{equation}
\mathfrak{C}_{p}\left(\tau\right)=\inf_{g\in\mathcal{G}_{b}}\frac{H_{g}(p)}{g(1/\tau)}.
\end{equation}
\end{theorem}

The reformulation of $\mathfrak{C}_p$ provided by Theorem \ref{thm:inf_form} gives us a great diversity of possible upper bounds on $\mathfrak{C}_{p}$. 
The following table presents different upper bounds obtained for different choices of functions.  
\begin{table}[h]
    \centering
    \begin{tabular}{l l l l}
\hline 
Name & $g$ & $\mathfrak{C}_{p}\left(\tau\right)$ upper bound & Related entropy\tabularnewline
\hline 
Cardinal (or Hartley) & $g(x)=x$ & $\text{card}\left(\M\right)\tau$ & $\text{Ent}_{1}(p)=\log\left(\text{card}\left(\M\right)\right)$\tabularnewline
Shannon & $g(x)=\log x$ & $\frac{\text{Ent}_{0}(p)}{\log(1/\tau)}$ & $\text{Ent}_{0}(p)=\sum p_{m}\log\left(1/p_{m}\right)$\tabularnewline
$\alpha-$Renyi  & $g(x)=x^{1-\alpha}$ & $\left(\tau e^{\text{Ent}_{\alpha}(p)}\right)^{1-\alpha}$ & $\text{Ent}_{\alpha}(p)=\frac{1}{1-\alpha}\log\sum p_{m}^{\alpha}$\tabularnewline
$\alpha-$Bertrand & $g(x)=x^{1-\alpha}\log^{\alpha}(x)$ & $\frac{\tau^{1-\alpha}}{\log^{\alpha}(1/\tau)}\sum_{m\in\M}\left(p_{m}\log\left(1/p_{m}\right)\right)^{\alpha}$ & Na\tabularnewline
\hline 
\end{tabular}
    \caption{Upper bounds on $\mathfrak{C}_{p}\left(\tau\right)$. Note that the parameter $\alpha$ is in $(0,1)$ and that Shannon and Bertrand's upper bounds are verified only for $p\in\mathcal{P}_{1/e}$.}
    \label{tab:EvsEntropy}
\end{table}

The bound based on the cardinality of $\mathcal{M}$ is a classical one and suffers from the curse of dimensionality when the cardinality is too large. The Shannon bound is cardinal-free and adapts with the entropy of $p$. Therefore, even if the cardinal scales exponentially in the dimension, when the entropy is low, the corresponding bound is more relevant than the classical bound, all the more so as when $\tau$ is large (this dependence being only logarithmic). The Renyi bound, obtained with $g(x)=x^{1-\alpha}$, is a good compromise between the two previous ones: it is smaller than the cardinal for large $\tau$ and decreases rapidly as $\tau$ decreases. 
\begin{remark}[Rényi entropy]
Depending on the considered $\tau$, upper-bounds provided in Table \ref{tab:EvsEntropy} may be more or less relevant. 
First, 
note that the first three upper bounds (Hartley-Shannon-Renyi) of Table \ref{tab:EvsEntropy} are informative, i.e.\ strictly less than 1, if and only if
\begin{equation}
\label{eq:tau_entropy}
    \tau< e^{-\text{Ent}_\alpha(p)},
\end{equation}
for $\alpha\in[0,1]$, where Rényi's entropy is defined by
\begin{equation}
\text{Ent}_{\alpha}\left(p\right):=\frac{1}{1-\alpha}\log\left(\sum_{m\in\M}p_{m}^{\alpha}\right).
\label{eq:def:Renyi}
\end{equation}
Note that Shannon and Hartley's entropies can be reformulated as limiting cases of Renyi's entropy \citep{renyi1961measures} when $\alpha=1$ and $\alpha =0$. 
Note also that all of these entropies are one when the distribution $p$ of the missing patterns is uniform. As soon as the latter is non-uniform, different regimes for these entropies can be identified. Indeed, 
for very small $\tau$ (less than $\min_m p_m$), Hartley's bound (i.e.\ the cardinal-type bound) is the lowest one. For larger $\tau$, Rényi's bound is bounded from above by Hartley's one (i.e.\ the cardinal-type bound) and from below by Shannon's one.
Furthermore, remark that Renyi's Entropy is non-increasing in $\alpha$ (see \cite{renyi1961measures}), so given \eqref{eq:tau_entropy}, as $\tau$ decreases, the Shannon's bound is the first one to be informative (less than 1), followed by Rényi's one, in turn, followed by Hartley's one.
The advantage of an entropic form is that you can use the additivity property which is very useful for dealing with examples. 
\end{remark}
\begin{remark}
A number of properties other than \cref{thm:inf_form} are very useful for dealing with certain distributions that have a particular structure (for example defined as a tensor product). 
\end{remark}

The proof of \cref{thm:inf_form} is based on the following lemma. 
\begin{lemma}\label{lemma:minUpperBound}
Let $b>0$, $\tau,p\in (0,b)$ and $g\in\mathcal{G}_b$, one has 
\begin{equation}
    \min(p,\tau)\leq \frac{pg(1/p)}{g(1/\tau)}. 
\end{equation}
\end{lemma}
\begin{proof} \hfill \\
 \begin{itemize}
 \item If $p< \tau$,
     \begin{align*}
\min(\tau,p) & =p\\
 & =\frac{pg(1/p)}{g(1/p)}\\
 & \leq\frac{pg(1/p)}{g(1/\tau)} & \text{using that }1/\tau<1/p\text{ and condition }(G_{1}).
\end{align*}
 \item If $\tau\leq p$,
 \begin{align*}
\min(\tau,p) & =\tau\\
 & =\frac{\tau g(1/\tau)}{g(1/\tau)}\\
 & \leq\frac{pg(1/p)}{g(1/\tau)} & \text{using that }\tau\leq p\text{ and condition }(G_{2}).
\end{align*}
 \end{itemize}
\end{proof}
\cref{thm:inf_form} is just an application of this lemma for each term of $\mathfrak{C}_p(\tau)$ in \eqref{def:E}.

\subsubsection{Some properties of $\mathfrak{C}_{p}\left(\tau\right)$}

\begin{propo}\label{tb:Emisc}
\hfill \\
\begin{enumerate}
\item $\mathfrak{C}_{p}$ is non-decreasing, concave, and for all $\lambda>1$
and $\tau\in(0,1)$:
\begin{equation}\label{ineq:Evsconstante}
  \mathfrak{C}_{p}\left(\lambda\tau\right)\leq\lambda\mathfrak{C}_{p}\left(\tau\right).  
\end{equation}
\item For $\tau\in(0,1)$:
\begin{equation}\label{eq:lowerE_p}
    \tau\leq \mathfrak{C}_p(\tau). 
\end{equation}

\item Let $p,q$ be two distributions with countable supports, for all $\tau\in(0,1)$,
\begin{equation}
\mathfrak{C}_{p\otimes q}\left(\tau\right)\leq\mathfrak{C}_{p}\left(\mathfrak{C}_{q}\left(\tau\right)\right).\label{ineq:Evstensorproduct}
\end{equation}
\item ``Data processing inequality'': Let $f:\M\longrightarrow\M'$, we
denote by $p^{f}$ the distribution of $f(M)$ when $M\sim p.$ We
have for all $\tau\in(0,1]$
\begin{equation}
\mathfrak{C}_{p^{f}}\left(\tau\right)\leq\mathfrak{C}_{p}\left(\tau\right).\label{ineq:Data_proc}
\end{equation}

\end{enumerate}
\end{propo}

The first inequality reads backward, it is less expensive to increase the argument than to increase the factor before $\mathfrak{C}_{p}$. This inequality is illustrated in \cref{thm:borne_reg_threhold} with the optimal choice of threshold. The second inequality \eqref{eq:lowerE_p} gives us a lower bound. Inequalities \eqref{ineq:Evstensorproduct} and \eqref{ineq:Data_proc} will help us to deal with examples that involve several combined processes of missing data generation such as the database merge model of \cref{sec:database}.

\begin{proof}
 \begin{itemize}
\item Proof of \eqref{ineq:Evsconstante}: For $\lambda\geq1$, $\min(p_{m},\lambda\tau)\leq\lambda\min(p_{m},\tau)$,
this conclude that $\mathfrak{C}_{p}(\lambda\tau)\leq\lambda\mathfrak{C}_{p}(\tau).$
\item Proof of \eqref{eq:lowerE_p}: We use \eqref{ineq:Evsconstante} with $\lambda=1/\tau>1.$
\item Proof of \eqref{ineq:Evstensorproduct}: 
\begin{align*}
\mathfrak{C}_{p\otimes q}(\tau) & =\sum_{m,m'}\min(p_{m}q_{m'},\tau)\\
 & =\sum_{m}\sum_{m'}q_{m'}\min\left(p_{m},\frac{\tau}{q_{m'}}\right)\\
 & =\sum_{m}\sum_{m'}q_{m'}\min\left(p_{m},\min\left(\frac{\tau}{q_{m'}},1\right)\right) & \text{because }p_{m}\leq1\\
 & \leq\sum_{m}\min\left(p_{m},\sum_{m'}q_{m'}\min\left(\frac{\tau}{q_{m'}},1\right)\right) & \text{using Jensen inequality}\\
 & \leq\mathfrak{C}_{p}\left(\mathfrak{C}_{q}(\tau)\right) & \text{using definition.}
\end{align*}
\item Proof of \eqref{ineq:Data_proc}: We will use $\min(a+b,c)\leq\min(a,c)+\min(b,c)$ for
$a,b,c\geq0.$
\begin{align*}
\mathfrak{C}_{p^{f}}(\tau) & =\sum_{k\in\text{Supp}(p^{f})}\min\left(p_{k}^{f},\tau\right)\\
 & \leq\sum_{k\in\text{Supp}(p^{f})}\min\left(\sum_{m:f(m)=k}p_{m},\tau\right)\\
 & \leq\sum_{k\in\text{Supp}(p^{f})}\sum_{m:f(m)=k}\min\left(p_{m},\tau\right)\\
 & \leq\sum_{m\in\text{Supp}(p)}\min\left(p_{m},\tau\right)\\
 & \leq\mathfrak{C}_{p}(\tau).
\end{align*}
\end{itemize}
\end{proof}

\subsection{Bernoulli Model} 
It is assumed that the components of $M$ are independent, and for
$j\in[d],$ $M_{j}\sim\mathcal{B}(\epsilon_{j})$ where $\epsilon_{j}\in\left[0,1\right]$.
The distribution $p$ of missing value pattern is $p=\mathcal{B}(\epsilon_{1})\otimes\cdots\otimes\mathcal{B}(\epsilon_{d})$. Let's define $\bar{\epsilon}:=\frac{1}{d}\sum_{j=1}^{d}\epsilon_j$, the average proportion of missing values. When $\epsilon_{1}=\epsilon_{2}=\cdots=\epsilon_{d}=\bar{\epsilon} $, the model is homogenous, otherwise it is heterogenous.
\subsubsection{Numerical experiments.} \label{sec:bernoulli_num}
The quantity $\mathfrak{C}_p\left(\frac{d}{n}\right)$ can be compared graphically for different missing pattern distributions of the Bernoulli model. In particular, we have chosen $d=4$ and 

\begin{itemize}
    \item $p_A$: Homogeneous Bernoulli with $\bar{\epsilon}=0.5$,
    \item $p_B$: Homogeneous Bernoulli with $\bar{\epsilon}=0.15$,
    \item $p_C$: Heterogeneous Bernoulli with $\bar{\epsilon}=0.15$  ($\epsilon_1=0.3,\epsilon_2=0.1,\epsilon_3=0.05,\epsilon_4=0.05$),
    \item $p_D$: Homogeneous Bernoulli with $\bar{\epsilon}=0.10$.
\end{itemize}
Note that $p_A$ matches with the uniform distribution over all missing patterns.   
\begin{figure}[h]
    \centering
    \includegraphics[width=10cm]{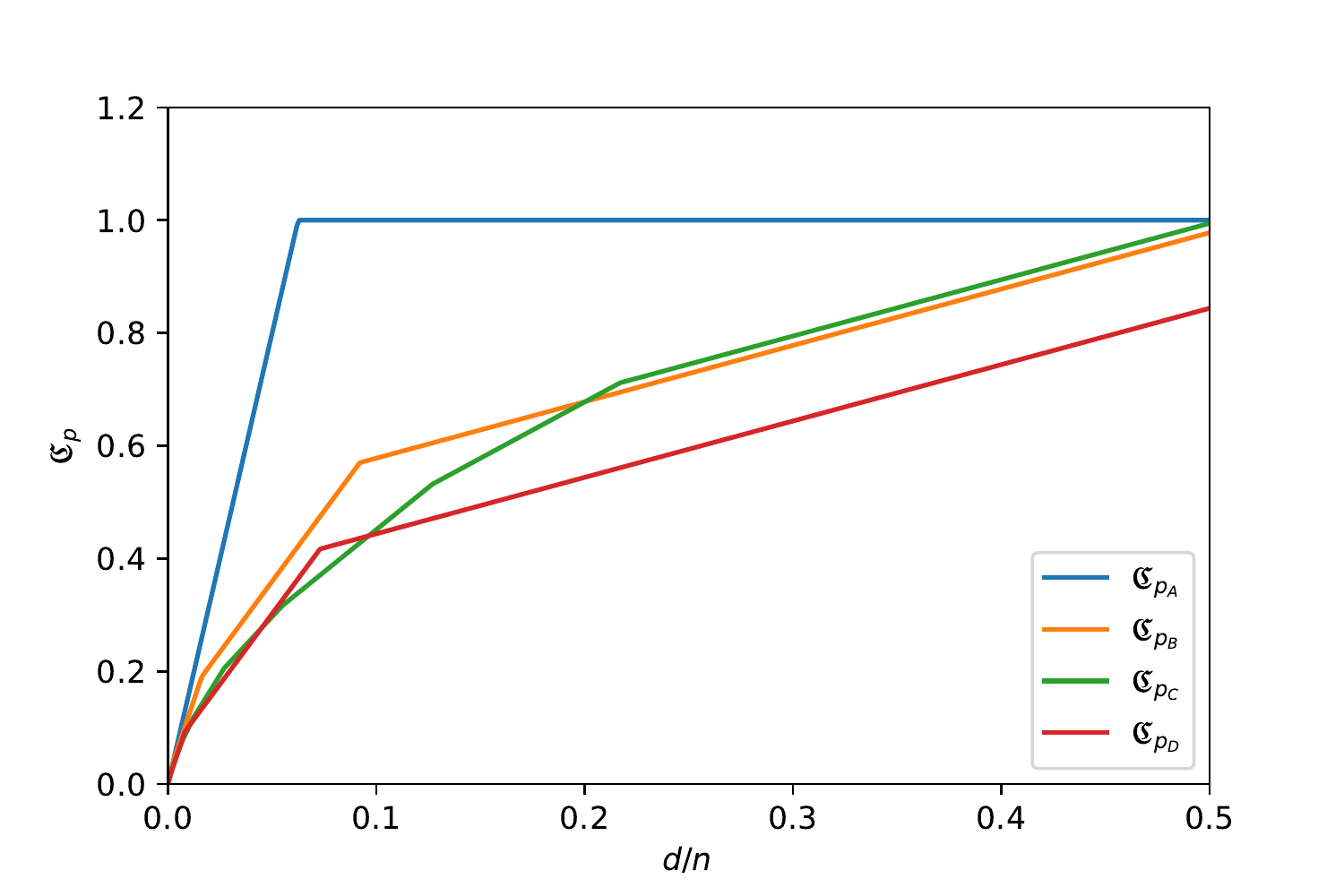}
    \caption{$\mathfrak{C}_{p}$ as a function of $\frac{d}{n}$. }
    \label{fig:diffE}
\end{figure}
\cref{fig:diffE} highlights three key points: 
\begin{enumerate}
    \item The distribution $p_A$, which corresponds to the uniform distribution on $\M$, is the worst in terms of complexity $\mathfrak{C}_{p}$.  
    \item The complexity seems to increase with the proportion of missing data $\bar{\epsilon}$ for homogeneous Bernoulli. 
    \item The comparison between homogeneous and heterogeneous does not seem relevant because $p_B$ and $p_C$ have the same proportion of missing values and each has a regime with a better $\mathfrak{C}_{p}$ than the other.
\end{enumerate}

\subsubsection{Proof of ~\cref{prop:BernoulliEBound} on the Homogeneous case}
\begin{proof}
From \eqref{eq:adaptabilityBernoulli}, we must bound $|\mathcal{B}_s|$ and $\delta_{s}$. For $|\mathcal{B}_s|$,
from \citep[Proposition 2.5]{massart2007concentration}

\begin{equation}\label{eq:sumBinom}
|\mathcal{B}_s|=\sum_{k=0}^{s}\binom{n}{k}\leq\left(\frac{ed}{s}\right)^{s}.
\end{equation}

Let $B\sim\mathcal{B}(\epsilon,d)$, we have $\delta_{s}=\P\left(B>s\right)$. Let
$t>0$, from Markov inequality
\begin{align*}
\P\left(B>s\right)=\P\left(t^{B}>t^{s}\right) & \leq\frac{\esp\left[t^{B}\right]}{t^{s}}\\
 & =\frac{\left(\epsilon t+(1-\epsilon)\right)^{d}}{t^{s}}\\
 & \leq\frac{\left(\epsilon t+1\right)^{d}}{t^{s}}\\
 & =\exp\left(d\log(1+\epsilon t)-s\log t\right)\\
 & \leq\exp\left(d\epsilon t-s\log t\right).
\end{align*}
The optimal choice $t=\frac{s}{\epsilon d}$, leads to 
\begin{equation}\label{eq:chernoffBin}
\P\left(B>s\right)\leq\epsilon^{s}\left(\frac{d}{s}\right)^{s}.
\end{equation}
Combining \eqref{eq:chernoffBin} and \eqref{eq:sumBinom}, we have, 
\[
\mathfrak{C}_{p}\left(\frac{d}{n}\right)\leq\left(\epsilon^{s}+\frac{d}{n}\right)\left(\frac{ed}{s}\right)^{s}.
\]
\end{proof}
\subsubsection{Heterogeneous Case}\label{sec:hetero}
Under a certain constraint, the same result as in the homogeneous case can be formulated for the heterogeneous case.
\begin{propo}\label{prop:BernoulliEBound2}
Under the assumptions of~\cref{thm:Risque_Complet}, and the Heterogeneous Bernoulli Model, if $$ s_{\bar\epsilon}\left(d/n\right)=1 \vee \left \lfloor\frac{\log\left(\frac{n}{d}\right)}{\log(\bar\epsilon^{-1})}\right\rfloor \wedge d \geq \bar\epsilon d,$$ then 
\begin{equation*}
\begin{split}
  \esp\left[\mathcal{E}\left(T_L\hat{f}^{(d/n)}\right)\right]\lesssim a_n \left(\frac{ed}{s_{\bar\epsilon}\left(d/n\right)}\right)^{s_{\bar\epsilon}\left(d/n\right)}\frac{d}{n}+ A_{\F_b}.
  \label{ineq:HomoBernoulliBestAlph}
\end{split}
\end{equation*}
\end{propo}
\begin{proof}[Proof of \ref{prop:BernoulliEBound2}]
Let's $\tau=d/n$, it is sufficient to show that 
\begin{equation*}
\mathfrak{C}_{p}\left(\tau\right)\leq\left(\frac{ed}{s_{\bar\epsilon}(\tau)}\right)^{s_{\bar\epsilon}(\tau)}\tau.
\end{equation*}
We remark that $\epsilon\longmapsto \text{Ent}_{\alpha}\left(\mathcal{B}(\epsilon)\right) =\frac{1}{1-\alpha}\log\left(\epsilon^{\alpha}+\left(1-\epsilon\right)^{\alpha}\right) $ is concave for $\alpha\in(0,1)$. 
Thus Renyi entropy of $p$ takes the form 
\begin{align*}
\text{Ent}_{\alpha}\left(p\right) & =\sum_{i=1}^n\text{Ent}_{\alpha}\left(\mathcal{B}(\epsilon_i)\right)& \text{using additivity of Renyi entropy}\\
& \leq d \text{Ent}_{\alpha}\left(\mathcal{B}(\bar\epsilon)\right) & \text{using Jensen inequality} \\
 & =\frac{d}{1-\alpha}\log\left(\bar\epsilon^{\alpha}+\left(1-\bar\epsilon\right)^{\alpha}\right),
\end{align*}
using Jensen inequality. From  Renyi's bound of Table \ref{tab:EvsEntropy}, for all $\alpha\in(0,1)$ this leads
to 
\begin{align*}
\mathfrak{C}_{p}\left(\tau\right) & \leq\left(\bar\epsilon^{\alpha}+\left(1-\bar\epsilon\right)^{\alpha}\right)^{d}\tau^{1-\alpha}\label{ineq:BernoulliHomoRenyi}\\
 & \leq\left(\bar\epsilon^{\alpha}+1\right)^{d}\tau^{1-\alpha}\\
 & \leq e^{\log(1+\bar\epsilon^{\alpha})d}\tau^{1-\alpha}\\
 & \leq e^{\bar\epsilon^{\alpha}d+\left(1-\alpha\right)\log(\tau)}.
\end{align*}
We can minimize function $\psi(\alpha)=\bar\epsilon^{\alpha}d+\left(1-\alpha\right)\log(\tau)$
on $\left(0,1\right)$: 
\[
\psi'(\alpha)=\log\left(\bar\epsilon\right)\bar\epsilon^{\alpha}d-\log\tau.
\]
The first order condition gives us that 
\begin{align*}
\log\left(\bar\epsilon\right)\bar\epsilon^{\alpha^{\star}}d-\log\tau & =0.
\end{align*}
And then,
\[
\bar\epsilon^{\alpha^{\star}}d=\frac{\log\tau}{\log\bar\epsilon}=s_{\bar\epsilon}(\tau).
\]
Thus, 
\[
\alpha^{\star}=\frac{\log\left(\frac{s_{\bar\epsilon}(\tau)}{d}\right)}{\log\left(\bar\epsilon\right)}.
\]
We have $\alpha^{\star}\in\left(0,1\right)$ if and only if $d\bar\epsilon<s_{\bar\epsilon}(\tau)<d$.
Under this condition, we have 
\begin{align*}
\psi(\alpha) & =s_{\bar\epsilon}(\tau)+\log(\tau)\left(1-\frac{\log\left(\frac{s_{\bar\epsilon}(\tau)}{d}\right)}{\log\left(\bar\epsilon\right)}\right)\\
 & =s_{\bar\epsilon}(\tau)+\frac{\log(\tau)}{\log(\bar\epsilon)}\left(\log\left(\bar\epsilon\right)-\log\left(\frac{s_{\bar\epsilon}(\tau)}{d}\right)\right)\\
 & =s_{\bar\epsilon}(\tau)\left(1+\log\left(\frac{d\bar\epsilon}{s_{\bar\epsilon}(\tau)}\right)\right)\\
 & =s_{\bar\epsilon}(\tau)\log\left(\frac{ed\bar\epsilon}{s_{\bar\epsilon}(\tau)}\right).
\end{align*}
The upper bound is therefore 
\begin{equation}
\mathfrak{C}_{p}\left(\tau\right)\leq\left(\frac{ed\bar\epsilon}{s_{\bar\epsilon}(\tau)}\right)^{s_{\bar\epsilon}(\tau)}=\left(\frac{ed}{s_{\bar\epsilon}(\tau)}\right)^{s_{\bar\epsilon}(\tau)}\bar\epsilon^ {s_{\bar\epsilon}(\tau)}=\left(\frac{ed}{s_{\bar\epsilon}(\tau)}\right)^{s_{\bar\epsilon}(\tau)}\tau.\label{ineq:HomoBernoulliBestAlpha}
\end{equation}
\end{proof}
\subsection{Proof of \cref{prop:datamerge}}\label{proof:database}
\begin{proof}
We denote by $p_P$ (resp. $p_N$) the distribution of $P$ (resp. $N$). Using \eqref{ineq:Data_proc}, we have 
\[ \mathfrak{C}_p\leq \mathfrak{C}_{p_H\otimes p_N}. \]
Furthermore, \eqref{ineq:Evstensorproduct} leads to,
\begin{align*}
    \mathfrak{C}_p \left( \frac{d}{n}\right)&\leq \mathfrak{C}_{p_H}\left(\mathfrak{C}_{p_N}\left( \frac{d}{n}\right)\right)\\
    &\leq h\mathfrak{C}_{p_N}\left( \frac{d}{n}\right) & \text{because  } |\rm{Supp}(H)|\leq h\\
    &\leq\left(\frac{ed}{s_{\eta}\left(d/n\right)}\right)^{s_{\eta}\left(d/n\right)}h\frac{d}{n},
\end{align*}
using \cref{prop:BernoulliEBound} on $\mathfrak{C}_{p_N}$.

\end{proof}
\subsection{Proof of \cref{lem:Cp_characterization}}\label{proof:Cp_characterization}
\begin{proof}
Let $\mathcal{B}$ be a subset of $\M$, we have
\begin{align*}
\mathfrak{C}_{p}\left(\frac{d}{n}\right) & =\sum_{m\in\M}p_{m}\wedge\left(\frac{d}{n}\right)\\
 & =\sum_{m\in\mathcal{B}}p_{m}\wedge\left(\frac{d}{n}\right)+\sum_{m\in\mathcal{B}^{c}}p_{m}\wedge\left(\frac{d}{n}\right)\\
 & \leq\sum_{m\in\mathcal{B}}\frac{d}{n}+\sum_{m\in\mathcal{B}^{c}}p_{m}\\
 & \leq\left|\mathcal{B}\right|\frac{d}{n}+\P\left(M\in\mathcal{B}\right).
\end{align*}
We obtain equality with $\mathcal{B} =\{m\in\M, p_m>d/n\}$. Thus, 
\begin{equation*}
    \mathfrak{C}_p\left(\frac{d}{n}\right)=\inf_{\mathcal{B}\subset\M}\left\{{\rm{Card}}(\mathcal{B})\frac{d}{n}+\P\left(M\in \mathcal{B}^c\right) \right\}.
\end{equation*}
\end{proof}
\section{Proof of Section \ref{sec:minimax}}
The purpose of this part is to establish the lower bounds of~\cref{sec:minimax}. 
\subsection{Preliminary lemmas.}
We consider a set of identifiable models:

\[
\text{\ensuremath{\mathcal{P}_{\mathcal{I}}:=\left\{ \P_{\mu},\mu\in\mathcal{I}\right\} ,}}
\]
where $\P_{\mu}$ is identifiable and $\mathcal{I}$ is a set of parameters.
Let $X_{1},...,X_{n}$ be i.i.d.\  observations of $\P_{\mu}$. We define the quadratic risk of an estimator $\widehat{\mu}$ as:
\begin{equation}
    r(\mu,\widehat{\mu}):=\esp_{\mathcal{P_\mu}}\left[\left(\widehat{\mu}_{n}-\mu\right)^{2}\right].
\end{equation}
The first step is to lower bound the integrated quadratic risk according to a distribution $\Pi$ on the set of parameters. 
\begin{lemma}\label{lem:gauss_mini}
We consider the class of models \[
\text{\ensuremath{\mathcal{P}:=\left\{ \P_{\mu}\sim\mathcal{N}(\mu,\sigma^2),\mu\in\mathcal{\R}\right\} }},
\]
with $\sigma^2$ known.
Let $\lambda>0$ and consider $\Pi\sim \mathcal{N}(0,\lambda^2)$ as a prior distribution for $\mu$. Then 
\begin{equation}
\label{eq:gauss_bayes_mini}
    \inf_{\widehat{\mu}}\esp_{\mu\sim\Pi}\left[r\left(\widehat{\mu},\mu\right)\right]=\frac{\lambda^{2}\sigma^{2}}{\sigma^{2}+\lambda^{2}n},
\end{equation}
where the infinimum is over all $\sigma(X_1,...,X_n)$-measurable estimator  $\widehat{\mu}$. 
\end{lemma}
\begin{proof}
 \begin{align*}
\inf_{\widehat{\mu}}\esp_{\mu\sim\Pi}\left[r\left(\widehat{\mu},\mu\right)\right] & =\inf_{\widehat{\mu}}\esp_{\mu\sim\Pi}\left[\esp_{\mu}\left[\left(\widehat{\mu}-\mu\right)^{2}\right]\right]\\
 & =\inf_{\widehat{\mu}}\esp\left[\esp\left[\left(\widehat{\mu}-\mu\right)^{2}|X_{1},...,X_{n}\right]\right]\\
 & =\esp\left[\left(\esp\left[\mu|X_{1},...,X_{n}\right]-\mu\right)^{2}\right]\\
 & =\mathbb{V}\left[\esp\left[\mu|X_{1},...,X_{n}\right]\right],
\end{align*}
because Bayes estimator $\esp\left[\mu|X_{1},...,X_{n}\right]$ is
optimal for the integrated Risk and unbiased. According prior $\Pi$,
$\left(\mu,X_{1},...,X_{n}\right)$ is a gaussian vector with the
following covariance matrix,
\[
\Gamma=\left(\begin{array}{cccc}
\lambda^{2} & \lambda^{2} & \cdots & \lambda^{2}\\
\lambda^{2} & \lambda^{2}+\sigma^{2} & \cdots & \lambda^{2}\\
\vdots & \vdots & \ddots & \vdots\\
\lambda^{2} & \lambda^{2} & \cdots & \lambda^{2}+\sigma^{2}
\end{array}\right).
\]
Thus, the variance of $\esp\left[\mu|X_{1},...,X_{n}\right]$ is 
\begin{align*}
\mathbb{V}\left[\esp\left[\mu|X_{1},...,X_{n}\right]\right] & =\lambda^{2}-\left(\lambda^{2},\cdots,\lambda^{2}\right)\left(\begin{array}{ccc}
\lambda^{2}+\sigma{{}^2} & \cdots & \lambda^{2}\\
\vdots & \ddots & \vdots\\
\lambda^{2} & \cdots & \lambda^{2}+\sigma{{}^2}
\end{array}\right)^{-1}\left(\begin{array}{c}
\lambda^{2}\\
\vdots\\
\lambda^{2}
\end{array}\right)\\
 & =\lambda^{2}-\left(\lambda^{2},\cdots,\lambda^{2}\right)\left(\sigma^{2}I_{n}+\lambda uu^{T}\right)^{-1}\left(\begin{array}{c}
\lambda^{2}\\
\vdots\\
\lambda^{2}
\end{array}\right),
\end{align*}
where $u=(1,...,1)^{T}$. The  Sherman-Morrison formula (see \citep{petersen2008matrix} for example) gives 
\begin{align*}
\left(\sigma^{2}I_{n}+\lambda uu^{T}\right)^{-1} & =\frac{1}{\sigma^{2}}I_{n}+\frac{\lambda^{2}uu^{T}/\sigma^{2}}{1+\frac{\lambda^{2}u^{T}u}{\sigma^{2}}}\\
 & =\frac{1}{\sigma^{2}}I_{n}+\frac{\lambda^{2}uu^{T}}{\sigma^{2}+\lambda^{2}n}.
\end{align*}
Thus,
\begin{align*}
\mathbb{V}\left[\esp\left[\mu|X_{1},...,X_{n}\right]\right] & =\lambda^{2}-\left(\frac{n\lambda^{4}}{\sigma^{2}}+\frac{n^{2}\lambda^{4}}{\sigma^{2}+\lambda^{2}n}\right)\\
 & =\frac{\lambda^{2}\sigma^{2}}{\sigma^{2}+\lambda^{2}n}.
\end{align*}
Thus, 
\begin{equation*}
\inf_{\widehat{\mu}}\esp_{\mu\sim\Pi}\left[\esp_{\mu}\left[\left(\widehat{\mu}-\mu\right)^{2}\right]\right]=\frac{\lambda^{2}\sigma^{2}}{\sigma^{2}+\lambda^{2}n}.
\end{equation*}
\end{proof}

\begin{remark}
Using the comparison between minimax and Bayes risks, this result can be used to prove that
\begin{equation*}
    \inf_{\widehat{\mu}}\sup_{\mu\in\R}r\left(\widehat{\mu},\mu\right)\geq \frac{\lambda^{2}\sigma^{2}}{\sigma^{2}+\lambda^{2}n}.
\end{equation*}
We obtain the classical result of the minimax estimation of a Gaussian mean where $\lambda\to \infty$:
\begin{equation*}
    \inf_{\widehat{\mu}}\sup_{\mu\in\R}r\left(\widehat{\mu},\mu\right)=\frac{\sigma^{2}}{n}. 
\end{equation*}
\end{remark}
Note that this lower bound is only valid when there are no constraints on the parameter space. However, we are interested in guarantees when $\mu$ is bounded, this is the purpose of the following result.
\begin{lemma}\label{lem:comparison}
Let $\lambda>0$ and $\Pi\sim \mathcal{N}(0,\lambda^2)$. Then
\begin{equation}\label{eq:comparison_truncated}
    \left|\esp_{\mu\sim\Pi}[r\left(\mu,T_{R}\widehat{\mu}\right)]-\esp_{\mu\sim\Pi}[r\left(T_{R}\mu,T_{R}\widehat{\mu}\right)]\right|\leq8\lambda^{2}e^{-\frac{1}{4}\left(\frac{R}{\lambda}\right)^{2}}.
\end{equation}
\end{lemma}
\begin{proof}
\begin{align}
\begin{split}\left|\esp_{\mu\sim\Pi}r\left(\mu,T_{R}\widehat{\mu}\right)-\esp_{\mu\sim\Pi}r\left(T_{R}\mu,T_{R}\widehat{\mu}\right)\right| 
 & \leq\begin{aligned}[t]
\left|\left(T_{R}\widehat{\mu}+R\right)^{2}\Pi\left[\mu<-R\right]+\left(T_{R}\widehat{\mu}-R\right)^{2}\Pi\left[\mu>R\right]\right|\\+\left|\int_{\left|\mu\right|>R}\left(T_{R}\widehat{\mu}+\mu\right)^{2}d\Pi\right|\nonumber
 \end{aligned}
\end{split} \\
 & \leq\int_{\left|\mu\right|>R}\left(6R^{2}+2\mu^{2}\right)d\Pi\nonumber \\
 & \leq8\esp_{\Pi}\left[\ind_{\left|\mu\right|>R}\mu^{2}\right]\label{eq:compa_1}\\
 & \leq8\sqrt{\Pi\left(\left|\mu\right|>R\right)\esp_{\Pi}\left[\mu^{4}\right]}\label{eq:compa_2}\\
 & \leq8\sqrt{e^{-\frac{R^{2}}{2\lambda^{2}}}\times3\frac{\sigma^{4}}{n{{}^2}}}\label{eq:compa_3}\\
 & \leq8\lambda^{2}e^{-\frac{1}{4}\left(\frac{R}{\lambda}\right)^{2}}\nonumber.
\end{align}
We have used Cauchy Schwarz inequality in \eqref{eq:compa_2}, moment and tail upper
bound of Gaussian distribution in \eqref{eq:compa_3}.
\end{proof}

\subsection{Minimax estimation of a value per missing pattern}
We consider the following Problem, 

\begin{equation}\label{prob:onlyM}
    Y=f^\star(M) +\epsilon,
\end{equation}
with $f^\star$ a deterministic function of the missing pattern M. We define $\widetilde{\mathcal{P}}_{p}(\sigma,R)$ as the set of $\P$ that satisfies:
\begin{enumerate}
    \item $\P(M=m)=p_m$.
    \item $\epsilon\sim \mathcal{N}(0,\sigma^2)$ and $\epsilon$ is independent of $M$. 
    \item $\max_{m\in\M}|f^\star(m)|\leq R$
\end{enumerate}
We denote by $\P_f$ the probability that satisfies the two first conditions with $f^\star=f$. We have the following minimax result on the estimation of $f^{\star}$. 
\begin{propo}\label{thm:minimax_onlyMP}
Let $R,\sigma,c>0$ such that $c\leq 16e^{-\frac{1}{4}\left(\frac{R}{\sigma}\right)^{2}} $, then 
\begin{equation}
    \inf_{\hat{f}}\sup_{\P\in\widetilde{\mathcal{P}}(\sigma,R)}\esp\left[\left(\widehat{f}(M)-f^\star(M)\right)^{2}\right]\geq(1-c)\sigma^{2}\mathfrak{C}_{p}\left(\frac{1}{n}\right).
\end{equation}
\end{propo}
\begin{proof}
 \textbf{Step 1}: Comparison with integrated risk and decomposition.

Let $\widehat{f}$ a estimator of $f^{\star}$. Without loss of generality,
we can assume that $\widehat{f}$ belongs to $B_{R}:=\left\{ f|\forall m\in\M,\left|f(m)\right|\leq R\right\} $.
Note $\Pi$ a prior distribution for $f^\star$.
\begin{align*}
\sup_{\P\in\widetilde{\mathcal{P}}(\sigma,R)}\esp_{\P}\left[\left(\widehat{f}(M)-f^\star(M)\right)^{2}\right] & =\sup_{f^{\star}\in B_{R}}\esp_{\P_{f^\star}}\left[\left(\widehat{f}(M)-f^\star(M)\right)^{2}\right]\\
 & \geq \esp_{f^{\star}\sim\Pi}\esp_{\P_{f^{\star}}}\left[\left(\widehat{f}(M)-T_R f^\star (M)\right)^{2}\right].
\end{align*}
We denote by $\esp_{\Pi}=\esp_{f^{\star}\sim\Pi}\esp_{\P_{f^{\star}}}$.
\begin{align}
\sup_{\P\in\mathcal{P}(\sigma,R)}\esp_{\P}\left[\left(\widehat{f}(M)-f^\star(M)\right)^{2}\right] & \geq\esp_{\Pi}\left[\left(\widehat{f}(M)-T_R f^\star(M)\right)^{2}\right]\nonumber \\
\begin{split}
  &\geq\begin{aligned}[t]
 -\left|\esp_{\Pi}\left[\left(\widehat{f}(M)-T_Rf^\star(M)\right)^{2}\right]
 -\esp_{\Pi}\left[\left(\widehat{f}(M)-f^\star(M)\right)^{2}\right]\right|
 \\+\esp_{\Pi}\left[\left(\widehat{f}(M)-f^\star(M)\right)^{2}\right].\label{eq:step1}\end{aligned}
\end{split}
\end{align}
\textbf{Step 2}: Lower bound of the first term.

We choose $\Pi=\bigotimes_{m\in\M}\mathcal{N}(0,\lambda_{m}^2)$ where
$0\leq\lambda_{m}\leq\sigma$. Conditioning by $M$ and using
Fubini theorem, we obtain, 
\begin{align}
\esp_{\Pi}\left[\left(\widehat{f}(M)-f^\star(M)\right)^{2}\right] & =\esp_{f^{\star}\sim\Pi}\sum_{m\in\M}p_{m}\esp_{\P_{f^{\star}}}\left[\left(\widehat{f}(m)-f^\star(m)\right)^{2}\right]\nonumber \\
 & =\sum_{m\in\M}p_{m}\esp_{\Pi}\left[\left(\widehat{f}(m)-f^\star(m)\right)^{2}\right]\nonumber \\
 & =\sum_{m\in\M}p_{m}\esp\left[\esp_{\Pi}\left[\left(\widehat{f}(m)-f^\star(m)\right)^{2}\right]|\left(M_{i}\right)_{i\in[n]}\right]\nonumber \\
 & \geq\sum_{m\in\M}p_{m}\esp\left[\esp_{\Pi}\left[\left(\esp\left[f^\star(m)|(Y_{i})_{i\in[n]}\right]-f^\star(m)\right)^{2}\right]|\left(M_{i}\right)_{i\in[n]}\right]\label{eq:variation_esp}\\
 & =\sum_{m\in\M}p_{m}\esp\left[\esp_{\Pi}\left[\left(\esp\left[f^\star(m)|(Y_{i})_{i\in E_{m}}\right]-f^\star(m)\right)^{2}\right]|\left(M_{i}\right)_{i\in[n]}\right]\label{eq:indep_condi}\\
 & =\sum_{m\in\M}p_{m}\esp\left[\esp_{\Pi}\left[\mathbb{V}\left[f^\star(m)|(Y_{i})_{i\in E_{m}}\right]\right]|\left(M_{i}\right)_{i\in[n]}\right].\nonumber 
\end{align}
We have used variational definition of $\esp\left[f(m)|(Y_{i})_{i\in[n]}\right]$
in \eqref{eq:variation_esp} and for a prior distribution $\Pi$, $Y_{i}$ and $Y_{j}$ are independent provided that
$M_{i}\neq M_{j}$ in \eqref{eq:indep_condi}. Using \cref{lem:gauss_mini}, we obtain, 
\[
\esp_{\Pi}\left[\left(\widehat{f}(M)-f^\star(M)\right)^{2}\right]\geq\sum_{m\in\M}p_{m}\esp\left[\frac{\lambda_{m}^{2}\sigma^{2}}{\sigma^{2}+\lambda_{m}^{2}\left|E_{m}\right|}\right].
\]
Using Jensen inequality (and \cref{lem:inverse_bernoulli}  with $|E_m| \sim \mathcal{B}(n,p_m)$), we have
\begin{equation}
\esp_{\Pi}\left[\left(\widehat{f}(M)-f^\star(M)\right)^{2}\right]\geq\sum_{m\in\M}p_{m}\frac{\lambda_{m}^{2}\sigma^{2}}{\sigma^{2}+\lambda_{m}^{2}np_{m}}.
\label{eq:step2}
\end{equation}
\textbf{Step 3}: Lower bound of the second term.

Using ~\cref{lem:comparison}, for $A=\left|\esp_{\Pi}\left[\left(\widehat{f}(M)-T_Rf^\star(M)\right)^{2}\right]-\esp_{\Pi}\left[\left(\widehat{f}(M)-f^\star(M)\right)^{2}\right]\right|$ , we have
\begin{align}
A & \leq\sum_{m\in\M}p_{m}\left|\esp_{\Pi}\left[\left(\widehat{f}(m)-T_Rf^\star(m)\right)^{2}\right]-\esp_{\Pi}\left[\left(\widehat{f}(m)-f^\star(m)\right)^{2}\right]\right|\nonumber \\
 & \leq\sum_{m\in\M}p_{m}8\lambda_{m}^{2}e^{-\frac{1}{4}\left(\frac{R}{\lambda_{m}}\right)^{2}}\nonumber \\
 & \leq\frac{c}{2}\sum_{m\in\M}p_{m}\lambda_{m}^{2},
 \label{eq:step3}
\end{align}
with $c\leq 16e^{-\frac{1}{4}\left(\frac{R}{\sigma}\right)^{2}} $ and since $\lambda_{m}\leq\sigma$.

\noindent\textbf{Step 4}: Choice of $\lambda_{m}$ and conclusion.

Combining \eqref{eq:step2} and \eqref{eq:step3}
in \eqref{eq:step1}, we obtain 
\[
\sup_{\P\in\widetilde{\mathcal{P}}(\sigma,R)}\esp_{\P}\left[\left(\widehat{f}(M)-f^\star(M)\right)^{2}\right]\geq\sum_{m\in\M}p_{m}\frac{\lambda_{m}^{2}\sigma^{2}}{\sigma^{2}+\lambda_{m}^{2}np_{m}}-\frac{c}{2}\sum_{m\in\M}p_{m}\lambda_{m}^{2}.
\]
We choose 
\[
\begin{cases}
\lambda^2_{m}=\frac{\sigma^{2}}{p_{m}n} & \text{if }p_{m}>1/n,\\
\lambda^2_{m}=\sigma^{2} & \text{if }p_{m}\leq1/n.
\end{cases}
\]
The condition $\lambda_{m}\leq\sigma$ holds and $p_{m}\lambda_{m}^{2}=\sigma^{2}\min\left(p_{m},1/n\right)$
, thus 
\begin{align*}
\sup_{\P\in\mathcal{P}(\sigma,R)}\esp_{\P}\left[\left(\widehat{f}(M)-f^\star(M)\right)^{2}\right] & \geq\sum_{m\in\M}p_{m}\frac{\sigma^{2}\min\left(p_{m},1/n\right)}{1+\min\left(p_{m},1\right)}-\frac{c}{2}\sigma^{2}\sum_{m\in\M}\min\left(p_{m},1/n\right)\\
 & \geq\sum_{m\in\M}p_{m}\frac{\sigma^{2}\min\left(p_{m},1/n\right)}{2}-\frac{c}{2}\sigma^{2}\sum_{m\in\M}\min\left(p_{m},1/n\right)\\
 & \geq\frac{\sigma^2}{2}\left(1-c\right)\mathfrak{C}_{p}\left(1/n\right).
\end{align*}
\end{proof}

\subsection{Proof of Section \ref{sec:minimax}} \label{proof:minimaxNAR}
\begin{proof}[Proof of ~\cref{thm:minimaxNAR}]
The idea is to reduce the prediction problem on class $\mathcal{P}_p(R,\sigma)$ to an estimation problem on class $\widetilde{\mathcal{P}}_p(R/d,\sigma)$ and then use \cref{thm:minimax_onlyMP}.
We denote by $m_0$ the missing pattern without missing values. 

Let $a\in[-1,1]^\M$, we consider $\P_a\in {\mathcal{P}}_p(R,\sigma)$ which satisfies: 
\begin{enumerate}
    \item $\beta_0=Ra_{m_0}$
    \item $\beta=R\left(1/d,...,1/d\right)^T$.
    \item For all $m\neq m_0$, $X|M=m\sim\delta_{\mu^{(m)}}$ where $\mu_{mis(m)}^{(m)}=(a_m-a_{m_0})(1,0,...,0)^T$ and $\mu_{\rm{obs}(m)}^{(m)}=0$ ($\delta$ denote the Dirac distribution).
\end{enumerate}
These problems satisfy Assumption \ref{ass:bornitude} with $\gamma =2$.

\noindent\textbf{Step 1}: Recall that the Bayes predictor is given by
\begin{align*}
f_{m}^{\star}\left(X_{\rm{obs}(m)}\right) & =\esp\left[Y|X_{\rm{obs}(m)},M=m\right]\\
&=\esp\left[\left\langle X,\beta\right\rangle |X,M=m\right]\\
&=\left\langle X_{\rm{obs}(m)},\beta_{\rm{obs}(m)}\right\rangle +\esp\left[\left\langle X_{mis(m)},\beta_{mis(m)}\right\rangle |X,M=m\right]
\end{align*}
Using $X|M=m\sim\delta_{\mu^{(m)}}$, we have  
\[
f_{m}^{\star}\left(X_{\rm{obs}(m)}\right)=Ra_m.\]
We have $f_{m}^{\star}$ $R-$lipschitz for $\ell_\infty$-norm (because $f_m$ is constant) and $\left|f_{m}^{\star}(0)\right|\leq R$ then $\P_a$ satisfies Assumption \ref{ass:Lipsch} with $B=R$ and $B^2(\gamma+1)\leq 3R^2$. 

\noindent\textbf{Step 2}: Problem reduction. 
For $\P_a$, $X_{\rm{obs}(M)}=0$ $\P_a$-a.s., then there are no information in $X_{\rm{obs}(M)}$, all the information is contained in the missing patterns and $Z=\left(X_{\rm{obs}(M)},M\right)=(0,M)$ $\P_a$-a.s.
i.e. we can ignore $X_{\rm{obs}(M)}$. The Bayes predictor is \begin{equation*}
    f^\star (M)=Ra_M,
\end{equation*}
and
\begin{equation*}
    Y= f^\star (M)+\epsilon.
\end{equation*}
This corresponds to Problem \eqref{prob:onlyM}, and varying $a\in[-1,1]^\M$, we obtain the set $\widetilde{\mathcal{P}}_p(R/d,\sigma)$. Thus, using \cref{thm:minimax_onlyMP}

\begin{align*}
\max_{\mathbb{P}\in{\mathcal{P}}_{p}(\sigma, R)}\esp_{\P}\left[\left(f^\star(Z)-\hat{f}(Z)\right)^{2}\right]&=  \max_{a\in[-1,1]^\M}\esp_{\P_a}\left[\left(f^\star(M)-\hat{f}(M)\right)^{2}\right]\\
&= \max_{\P\in\widetilde{\mathcal{P}}_p(R/d,\sigma)}\esp_{\P}\left[\left(f^\star(M)-\hat{f}(M)\right)^{2}\right]\\
&\geq (1-c)\frac{\sigma^{2}}{2}\mathfrak{C}_{p}\left(\frac{1}{n}\right).
\end{align*}
 
\end{proof}

\begin{proof}[Proof of \cref{cor:minimaxMAR}]

We will use the same method as in the previous proof. We need to find a subclass of problem MAR included in $\mathcal{P}_{p}(\sigma, R)$. Let $a\in[-1,1]^\M$, we denote by $\P_a$ the following problem.
\begin{enumerate}
    \item $X_1\sim\mathcal{N}(0,1)$.
    \item $M=h(X_1)$ a.s. where $h$ satisfies $\P\left(h(X_1)=m\right)=p_m$.
    \item $X_{2:d}|X_1\sim \delta_{\mu^{(h(X_1))}}$ where $\mu_{mis(m)}^{(h(X_1))}= a_{h(X_1)}(1,0,...,0)^T$ and $\mu_{(obs(m)}^{(h(X_1))}=0$.
    \item $\beta=R\left(0,1/(d-1),...,1/(d-1)\right)^T$.
\end{enumerate}
By construction, $\P_a$ is MAR, and Assumption \ref{ass:bornitude} holds with $\mm=1$. 
With this new choice of $\P_a$, the rest of the proof is similar to the proof of \cref{thm:minimaxNAR}. 
 

\end{proof}

\section{Details on numerical experiments of \cref{sec:experiments}}
The codes of our numerical experiments are all available in Github\footnote{https://github.com/AlexisAyme/minimax\_linear\_na}.
\subsection{Details on data generation setting}
In order for the simulations to be reproducible, here are the useful parameters to generate the dataset of \cref{sec:experiments}.

Let $U\in\R^{8\times 8}$ be the diagonal matrix per block with each block equal to $(1)_{i,j\in[2]}$. For all scenarios $\beta=(1,1,1,1,1,1,1,1)$ and $\beta_0=0$. 
\paragraph{(a) MCAR} $\mu=(1,1,1,1,1,1,1,1)$, $\Sigma= U$, and $\sigma=0.1$.

\paragraph{(b) MAR} $\Sigma= U$ on corresponding block and $\sigma=0.5$. 

\paragraph{(c) GPMM.} \label{sec:GMMdetails}
$(X,M)$ is distributed according to \cref{ass:ass4.1} with $\sigma=1$ and 
\begin{itemize}
    \item $p_{m_1}=0.6$, $m_1=(0,1,0,1,0,0,0,0)$, $\mu_{m_1}=(0,5,4,-1,0,0,0,0)$, and  $\Sigma_{m_1}=U$.
    \item $p_{m_2}=0.3$, $m_2=(1,0,1,1,0,0,0,0)$, $\mu_{m_2}=(1,3,0,2,0,0,0,0)$, and  $\Sigma_{m_2}= (1)_{i,j\in[8]}$.
    \item $p_{m_3}=0.02$, $m_3=(0,1,1,1,0,0,0,0)$, $\mu_{m_3}=(0,5,4,-1,0,0,0,0)$, and  $\Sigma_{m_3}=I_8$.
    \item $p_{m_4}=0.02$, $m_4=(1,1,0,1,0,0,0,0)$, $\mu_{m_4}=(0,5,0,-1,0,0,0,0)$, and  $\Sigma_{m_4}=I_8$.
    \item $p_{m_5}=0.02$, $m_5=(1,1,0,0,0,0,0,0)$, $\mu_{m_5}=(0,-10,7,-1,0,0,0,0)$, and  $\Sigma_{m_5}=I_8$.
    \item $p_{m_6}=0.02$, $m_6=(0,1,0,0,0,0,0,0)$, $\mu_{m_6}=(0,9,0,-1,0,0,0,0)$, and  $\Sigma_{m_6}=I_8$.
    \item $p_{m_7}=0.02$, $m_7=(0,0,1,0,0,0,0,0)$, $\mu_{m_7}=(3,0,0,-1,0,0,0,0)$, and  $\Sigma_{m_7}=I_8$.
    
\end{itemize}

\subsection{Training Time}\label{sec:learningTime}

\cref{fig:time_XP} corresponds to the training time of the simulations in \cref{sec:experiments} and are associated with the curve in \cref{fig:excess_risk_XP}. NeuMiss has a much more limiting training time than other methods. The most time-efficient method is also the most biased. Indeed, Cst-imp+LR does not adapt to any scenario (see \cref{fig:excess_risk_XP}). The training times are similar for the other methods, but MICE+LR is only relevant for scenario (a).

\begin{figure*}[h]
    \centering
    \begin{tabular}{ccc}
        \includegraphics[width=0.3\linewidth]{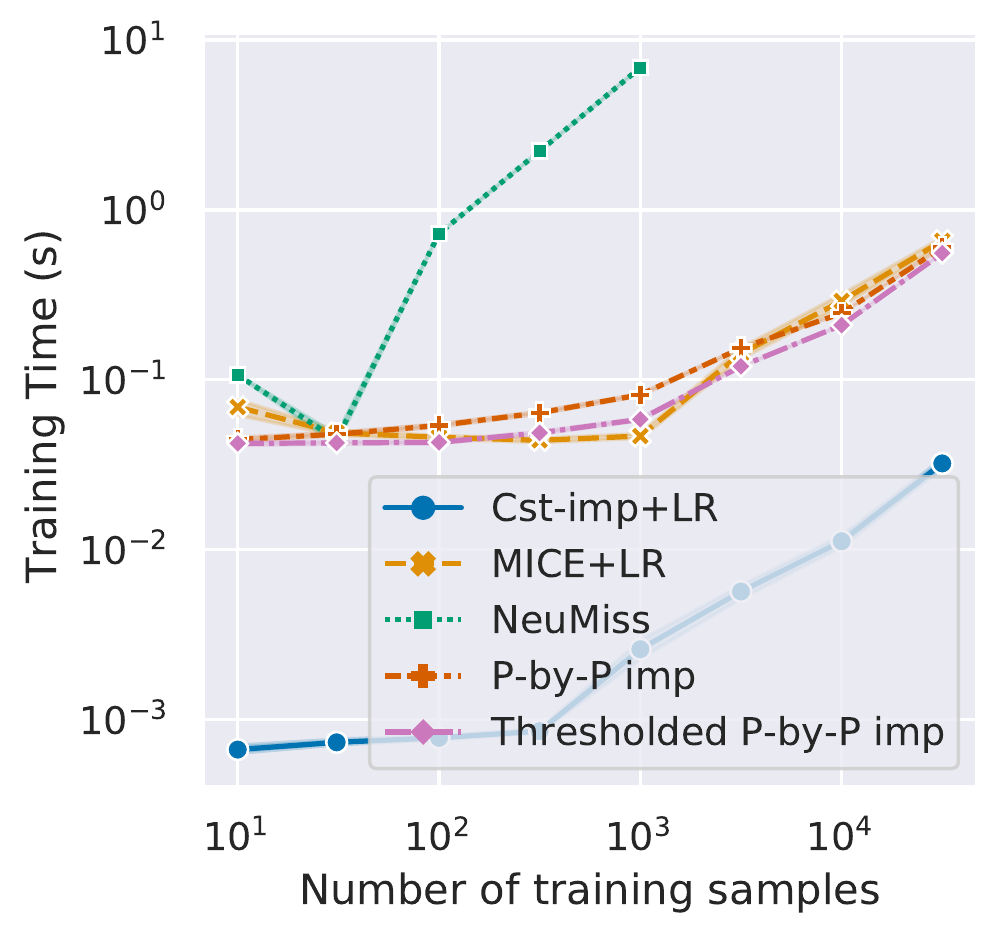}&
        \includegraphics[width=0.3\linewidth]{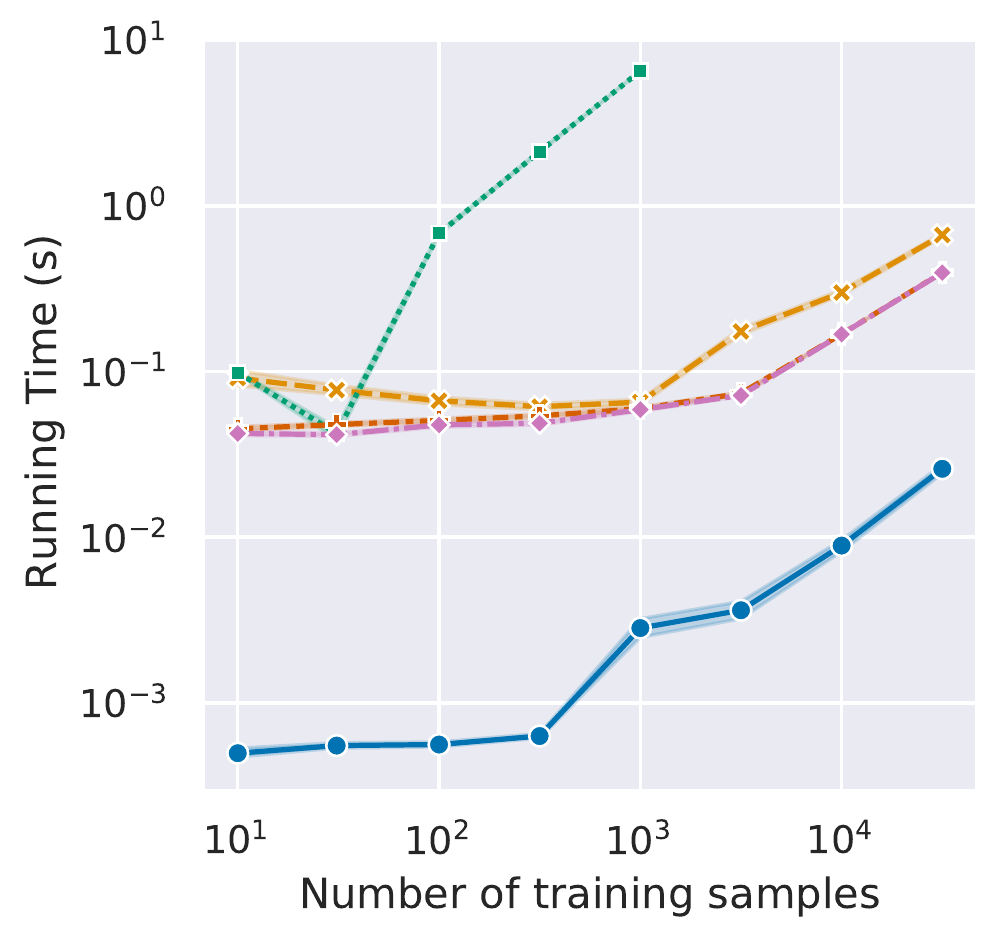}&
        \includegraphics[width=0.3\linewidth]{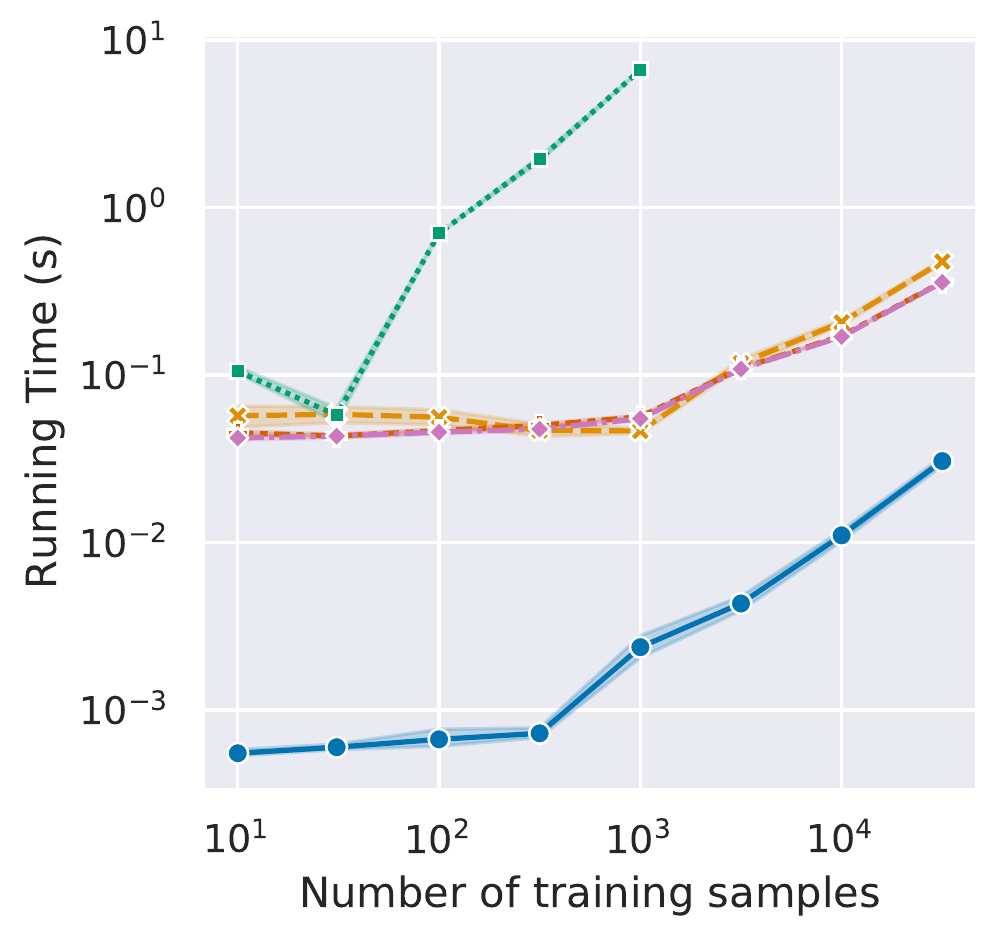} \\
        {(a) MCAR} & {(b) MAR} & {(c) MNAR-GPMM}
    \end{tabular}
    \caption{\label{fig:time_XP} Training time w.r.t.~the number of training samples.}
\end{figure*}

\end{document}